\documentclass{article}

     \PassOptionsToPackage{numbers}{natbib}

   \usepackage[final]{neurips_2019}

\usepackage[utf8]{inputenc} 
\usepackage[T1]{fontenc}    
\usepackage{hyperref}       
\usepackage{url}            
\usepackage{booktabs}       
\usepackage{amsfonts}       
\usepackage{nicefrac}       
\usepackage{microtype}      
\usepackage[english]{babel}
\usepackage{amsmath}
\usepackage{amsthm}
\usepackage{amssymb}
\usepackage[usenames,dvipsnames]{xcolor}
\usepackage{graphicx}
\usepackage{tikz}
\usepackage[colorinlistoftodos, color=orange!50]{todonotes}
\usepackage{fancybox}
\usepackage{epsfig}
\usepackage{soul}
\usepackage{enumitem}
\usepackage{subcaption}
\usepackage{algorithm}
\usepackage{algorithmic, eqparbox}
\usepackage{dsfont}
\usepackage{chngpage}
\usepackage{appendix}
\usepackage{capt-of}
\usepackage{newfloat}
\usepackage{mathtools}
\usepackage{relsize}

\author{%
  Etienne Boursier\\
  CMLA, ENS Paris-Saclay\\
  \texttt{etienne.boursier@ens-paris-saclay.fr} \\
   \And
   Vianney Perchet \\
   CMLA, ENS Paris-Saclay\\
   Criteo AI Lab, Paris \\
   \texttt{vianney.perchet@normalesup.org} \\
}
\newcommand{\algoone}[1][]{\textsc{sic-mmab}#1\ }
\newcommand{\algooneadapted}[1][]{\textsc{adapted sic-mmab}#1\ }
\newcommand{\algotwo}[1][]{\textsc{sic-mmab2}#1\ }
\newcommand{\ucb}{\textsc{ucb}\ }
\newcommand{\send}[1][]{\textbf{Send Protocol}#1\ }
\newcommand{\receive}[1][]{\textbf{Receive Protocol}#1\ }
\newcommand{\musicalchair}[1][]{MusicalChairs#1\ }
\newcommand{\estimm}[1][]{Estimate\_M#1\ }
\newcommand{\estimmnosens}[1][]{Estimate\_M\_NoSensing#1\ }

\newcommand{\sensingone}[1][]{Statistic Sensing#1\ }
\newcommand{\sensingtwo}[1][]{Collision Sensing#1\ }
\newcommand{\sensingoneshort}[1][]{Stat. Sensing#1\ }
\newcommand{\sensingtwoshort}[1][]{Col. Sensing#1\ }

\newcommand{\ie}{i.e.,\ }

\definecolor{brickred}{rgb}{0.8, 0.25, 0.33}
\newcommand{\tablecolor}{\color{brickred}}

\newcommand{\algothree}[1][]{\textsc{dyn-mmab}#1\ }

\newtheorem{lemm}{Lemma}
\newtheorem{thm}{Theorem}
\newtheorem{prop}{Proposition}
\newtheorem{hyp}{Assumption}

\DeclareFloatingEnvironment[
    fileext=los,
    name=Pseudocode,
    placement=tbhp,
]{pseudocode}

\title{SIC\,-\,MMAB: Synchronisation Involves Communication in Multiplayer Multi-Armed Bandits}

\allowdisplaybreaks 

\begin{document}
\maketitle
%
\begin{abstract}
Motivated by cognitive radio networks, we consider the stochastic multiplayer multi-armed bandit problem, where several players pull arms simultaneously and collisions occur if one of them is pulled by several players at the same stage.  We present a decentralized algorithm that achieves the same performance as a centralized one,  contradicting the existing lower bounds for that problem. This is possible by ``hacking'' the standard model by constructing a communication protocol between players that deliberately enforces collisions, allowing them to share their information at a negligible cost. 
This motivates the introduction of a more appropriate dynamic setting without sensing, where similar communication protocols are no longer possible. However, we show that the logarithmic growth of the regret is still achievable for this model with a new algorithm.
\end{abstract}

%

\section{Introduction}

In the stochastic Multi Armed Bandit problem (MAB),  a single player sequentially takes a decision (or ``pulls an arm'') amongst a finite set of possibilities  $[K]\coloneqq\{1,\ldots,K\}$. After pulling arm $k \in [K]$ at stage $t \in \mathds{N}^*$, the player receives a random reward $X_k(t) \in [0,1]$, drawn \textsl{i.i.d.}\ according to some unknown distribution $\nu_k$ of expectation $\mu_k \coloneqq \mathds{E}[X_k(t)]$. Her objective is to maximize her cumulative  reward up to stage $T \in \mathds{N}^*$. This sequential decision problem, first introduced for clinical trials \citep{thompson33, robbins52}, involves an ``\textsl{exploration/exploitation} dilemma'' where the player must trade-off acquiring \textsl{vs.}\ using information. The performance of an algorithm is controlled in term of \textbf{regret}, the difference of the cumulated reward of an optimal algorithm knowing the distributions $(\nu_k)_{k \in [K]}$ beforehand and the cumulated reward of the player. It is known that any ``reasonable'' algorithm must incur at least a logarithmic regret \citep{LaiRobbins}, which is attained by some existing algorithms such as \ucb \citep{agrawal95, Auer2002}.

MAB has been recently popularized thanks to its applications to online recommendation systems. Many different variants of MAB and classes of algorithms have thus emerged in the recent years \citep[see][]{survey}. In particular, they have been considered for cognitive radios \citep{jouini}, where the problem gets more intricate as multiple users are involved and they collide if they pull the same arm $k$ at the same time $t$, \ie they transmit on the same channel. If this happens, they all receive $0$ as a reward instead of $X_k(t)$, meaning that no message is transmitted.

If a central agent controls simultaneously all players' behavior then a tight lower bound is  known \citep{anantharam, centralized2}. Yet this centralized problem is not adapted to cognitive radios, as it allows communication between players at each time step; in practice, this induces  significant costs in both energy and time. As a consequence, most of the current interest lies in the decentralized case \citep{lower2, anandkumar, mega}, which presents another complication due to the feedback. Besides the received reward, an additional piece of information may be observed at each time step. When this extra observation is the collision indicator, \citet{musicalchair} provided two algorithms for both a fixed and a varying number of players. They are based on a \textit{Musical Chairs} procedure that quickly assigns players to different arms.
\citet{besson} provided an efficient UCB-based algorithm if $X_k(t)$ is observed instead\footnote{We stress that $X_k(t)$ does not necessarily correspond to the received reward in case of collision.}.
\citet{gabor} recently proposed an algorithm using no additional information. The performances of these algorithms and the underlying model differences are summarized in Table~\ref{table:comparatif}, Section~\ref{sec:models}.

The first non trivial lower bound for this problem has been recently improved \citep{lower2, besson}. These lower bounds suggest that decentralization adds to the regret a multiplicative factor $M,$ the number of players, compared to the centralized case \citep{anantharam}.
Interestingly, these lower bounds scale linearly with the inverse of the gaps between the $\mu_k$ whereas this scaling is quadratic for most of the existing algorithms.
This is due to the fact that although collisions account for most of the regret, lower bounds are proved without considering them.

Although it is out of our scope, the heterogeneous model introduced by \citet{diffmeans0} is worth mentioning. In this case, the reward distribution depends on each user \citep{differentmeans, diffmeans2}. An algorithm reaching the optimal allocation without explicit communication between the players was recently proposed \citep{GoT2018}.

Our main contributions are the following:
\vspace{-0.5em}
\paragraph{Section~\ref{sec:synchcomm}:}When collisions are observed, we introduce a new decentralized algorithm that is ``hacking'' the setting and induces communication between players  through deliberate collisions. The regret of this algorithm reaches asymptotically (up to some universal constant) the lower bound of the centralized problem, meaning that the aforementioned lower bounds are unfortunately incorrect.

This algorithm relies on the unrealistic assumption that all users start transmitting at the very same time. It also explains why the current literature fails to provide near optimal results for the multiplayer bandits. It therefore appears that the assumption of synchronization has to be removed for practical application of the multiplayer bandits problem. On the other hand, this technique also shows that exhibiting lower bounds in multi-player MAB is more complex than in stochastic standard MAB.
\paragraph{Section~\ref{sec:nosens2}:}Without synchronization or collision observations, we propose the first algorithm with a logarithmic regret. The dependencies in the gaps between rewards yet become quadratic.
%

\subsection{Models}
\label{sec:models}
In this section, we introduce different  models of multiplayer MAB with a known number of arms $K$ but an unknown number of players $M \leq K$. The horizon $T$ is assumed known to the players (for simplicity of exposure, as  the anytime generalization of results is now well understood \cite{degenne2016anytime}). At each time step $t \in [T]$, given their (private) information, all players $j \in [M]$ simultaneously pull the arms $\pi^j(t)$ and receive the reward $r^j(t) \in [0, 1]$ such that
\vspace{-0.5em}
\begin{align*} r^j(t) \coloneqq X_{\pi^j(t)}(t) (1 - \eta_{\pi^j(t)}(t)), \text{ where } \eta_{\pi^j(t)}(t) \text{ is the collision indicator defined by } \\ 
\eta_{k}(t) \coloneqq \mathds{1}_{\# C_k(t) > 1} \qquad \text{ with } \qquad C_k(t) \coloneqq \{j \in [M] \ |\ \pi^j(t) = k\}. \end{align*}
The problem is \textbf{centralized} if players can communicate any information to each other. In that case, they can easily avoid collisions and share their statistics. In opposition, the problem is \textbf{decentralized} when players have only access to their own rewards and actions.
The crucial concept we introduce is (a)synchronization between players. With synchronization, the model is called \textbf{static}. 

\begin{hyp}[Synchronization]
\label{hyp:synch}
Player $i$ enters the bandit game at the time $\tau_i=0$ and stays until the final horizon $T$. This is common knowledge to all players.
\end{hyp}

\begin{hyp}[Quasi-Asynchronization]
\label{hyp:asynch}
Players enter at different times $\tau_i \in \lbrace 0, \ldots, T-1 \rbrace$ and stay until the final horizon $T$. The $\tau_i$ are unknown to all players (including $i$).
\end{hyp}
With quasi-asynchronicity\footnote{We prefer not to mention asynchronicity as players still use shared discrete time slots.}, the model is \textbf{dynamic} and several variants already exist \citep{musicalchair}. Denote by $\mathbf{M}(t)$ the set of players in the game at time $t$ (unknown but not random) and by $\mu_{(n)}$ the $n$-th order statistics of $\mu$, \ie $\mu_{(1)} \geq \mu_{(2)} \geq \ldots \geq \mu_{(K)}$. The total regret is then defined for both static and dynamic models by:
\begin{small}
\begin{equation*}
R_T \coloneqq {\mathlarger\sum_{t=1}^T} {\mathlarger \sum_{k = 1}^{\#\mathbf{M}(t)}} \mu_{(k)} - \mathbb{E}_\mu \left[ {\mathlarger\sum_{t=1}^T} {\mathlarger\sum_{j \in \mathbf{M}(t)}} r^j(t) \right].
\end{equation*}
\end{small}
As mentioned in the introduction, different observation settings are considered.
\begin{description}
\item[{\sensingtwo[:]}] Player $j$ observes $\eta_{\pi^j(t)}(t)$ and $r^j(t)$ at each time step.
\item[{No sensing:}] Player $j$ only observes $r^j(t)$, \ie a reward of $0$ can indistinguishably come from a collision with another player or a null statistic $X_{\pi^j(t)}(t)$. 
\end{description}
Notice that as soon as $\mathbb{P}(X_k = 0) = 0$, the No Sensing and \sensingtwo settings are equivalent.
The setting where both $X_{\pi^j}(t)$ and $r^j(t)$ are observed is also considered in the literature and is called \textbf{\sensingone}\citep{besson}. The No Sensing setting is the most difficult one as there is no extra observation. 

Table~\ref{table:comparatif} below compares the performances of the major algorithms, specifying the precise setting considered for each of them. The second algorithm of \citet{gabor} and our algorithms also have problem independent bounds that are not mentioned in Table~\ref{table:comparatif} for the sake of clarity. Due to space constraints, \algooneadapted[,] \algotwo and their related results are presented in Appendix~\ref{sec:nosens1}.
Note that the two dynamic algorithms in Table~\ref{table:comparatif} rely on different specific assumptions.
\begin{table}[h]
\begin{adjustwidth}{-0.5in}{-0.5in}
\centering
\scriptsize{{\setlength{\extrarowheight}{5pt}
\begin{tabular}{|c|c|c|c|}
\hline
\textbf{Model} & \textbf{Algorithm's Reference} & \textbf{Prior knowledge} 
& \textbf{Asymptotic Upper bound (up to constant factor)} \\[5pt]
\hline
Centralized Multiplayer & \mbox{Theorem~1 \citep{centralized2} }
& $M$ 
&  ${\mathlarger \sum_{k > M}} \frac{\log(T)}{\mu_{(M)} - \mu_{(k)}}$ \\[8pt]
\hline
\hline
Decentralized, \sensingoneshort & \mbox{Theorem~11 \citep{besson}} &$M$ & $M^3 \!\!\!\! {\mathlarger\sum_{1 \leq i < k \leq K}} \frac{ \log(T)}{\left(\mu_{(i)}- \mu_{(k)}\right)^2}$  \\[10pt]
\hline 
\hline
Decentralized, \sensingtwoshort & \mbox{Theorem~1 \citep{musicalchair}
} & $\mu_{(M)} \! - \! \mu_{(M+1)}$
& $\frac{MK \log(T)}{\left(\mu_{(M)} - \mu_{(M+1)}\right)^2}$ \\[8pt]
\hline
Decentralized, \sensingtwoshort & \mbox{\tablecolor\algoone
}(Thm~\ref{thm:upperbound}) & - 
& {\tablecolor${\mathlarger\sum_{k > M}} \frac{\log(T)}{\mu_{(M)} - \mu_{(k)}} + MK \log(T)$} \\[8pt]
\hline
\hline
Decentralized, No Sensing & Theorem~1.1 \citep{gabor}
& $M$ 
& $\frac{MK \log(T)}{\left(\mu_{(M)} - \mu_{(M+1)}\right)^2}$ \\[10pt]
\hline
Decentralized, No Sensing & Theorem~1.2 \citep{gabor}
& $M, \mu_{(M)}$ 
&$ \frac{MK^2}{\mu_{(M)}} \log^2(T) + MK \frac{\log(T)}{\Delta'}$ \\[8pt]
\hline
Decentralized, No Sensing & \mbox{\tablecolor \textsc{adapt. sic-mmab}} (Eq~\eqref{eq:algooneadapted})& $\mu_{(K)}$
&{\tablecolor ${\mathlarger \sum_{k >M}} \frac{\log(T)}{\mu_{(M)} - \mu_{(k)}} +  \frac{M^3 K\log(T)}{\mu_{(K)}} \log^2 \big( \log(T) \big) 
$} \\[8pt]
\hline
Decentralized, No Sensing & \mbox{\tablecolor\algotwo}(Thm~\ref{thm:algo2}) & $\mu_{(K)}$
&${\tablecolor M {\mathlarger \sum_{k > M}} \frac{\log(T)}{\mu_{(M)}-\mu_{(k)}} + \frac{MK^2}{\mu_{(K)}} \log(T) }$ \\[8pt]
\hline\hline
Dec., \sensingtwoshort, Dynamic &\mbox{Theorem~2 \citep{musicalchair}} & $\bar{\Delta}_{(M)}$ 
& \rule{0pt}{2em} $\frac{M \sqrt{K \log(T) T}}{\bar{\Delta}_{(M)}^2}$\\[8pt]
\hline
Dec., No Sensing, Dynamic &\mbox{\tablecolor\algothree}(Thm~\ref{thm:dynamicregret}) & - 
& {\tablecolor $  \frac{MK \log(T)}{\bar{\Delta}_{(M)}^2}+ \frac{M^2K \log(T)}{\mu_{(M)}} $ }\\[8pt]
\hline
\end{tabular}}}
\end{adjustwidth}
\caption{\label{table:comparatif} Performances of different algorithms. 
Our algorithms and results are highlighted in red. $\bar{\Delta}_{(M)} \coloneqq \min_{i=1,..., M}(\mu_{(i)} - \mu_{(i+1)})$ is the smallest gap among the top-$M\!+\!1$ arms and $\Delta' \coloneqq \min \{ \mu_{(M)} - \mu_i \ | \ \mu_{(M)} - \mu_i > 0 \}$ is the positive sub-optimality gap.}
\end{table}
\vspace{-2.5em}

%

\section{\sensingtwo[:]achieving centralized performances by communicating through collisions}
\label{sec:synchcomm}
In this section, we consider the \sensingtwo static model and prove that the decentralized problem is almost as complex, in terms of regret growth, as the centralized one. When players are synchronized, we provide an algorithm with an exploration regret similar to the known centralized lower bound \citep{anantharam}. This algorithm strongly relies on the synchronization assumption, which we leverage to allow communication between players through observed collisions. The communication protocol is detailed and explained in Section~\ref{Phase comm}.
This result also implies that the two lower bounds provided in the literature \citep{besson, lower2} are unfortunately not correct. Indeed, the factor $M$ that was supposed to be the cost of the decentralization in the regret should not appear.

Let us now describe our algorithm \algoone[.]It consists of several phases.
\vspace{-0.5em}
\begin{enumerate}
\item The initialization phase first estimates the number of players and assigns ranks among them.
\item Players then alternate between exploration phases and communication phases.
\vspace{-0.5em}
\begin{enumerate}
\item During the $p$-th exploration phase, each arm is pulled $2^p$ times and its performance is estimated in a Successive Accepts and Rejects fashion \citep{perchet2013, SAR}.
\vspace{-0.3em}
\item During the communication phases, players communicate their statistics to each other using collisions. Afterwards, the updated common statistics are known to all players.
\end{enumerate}
\item The last phase, the exploitation one, is triggered for a player as soon as an arm is detected as optimal and assigned to her. This player then pulls this arm until the final horizon $T$.
\end{enumerate}

\subsection{Some preliminary notations}
Players that are not in the exploitation phase are called \textbf{active}. We denote, with a slight abuse of notation, by $[M_p]$ the set of active players during the $p$-th phase of exploration-communication and by $M_p \leq M$ its cardinality. Notice that $M_p$ is non increasing because players never leave the exploitation phase. \\
Any arm among the top-$M$ ones is called \textbf{optimal} and any other arm is \textbf{sub-optimal}. Arms that still need to be explored (players cannot determine whether they are optimal or sub-optimal yet) are \textbf{active}. We denote, with the same abuse of notation, the set of active arms by $[K_p]$ of cardinality $K_p \leq K$. By construction of our algorithm, this set is common to all active players at each stage.

Our algorithm is based on a protocol called \textsl{sequential hopping} \citep{sequentialhopping}. It consists of incrementing the index of the arm pulled by a specific player: if she plays arm $\pi_t^k$ at time $t$, she will play $\pi_{t+1}^k=\pi_t^k+1 \text{ (mod } [K_{p}])$ at time $t+1$ during the $p$-th exploration phase.

\subsection{Description of our protocol}

As mentioned above, the \algoone algorithm consists of several phases. During the communication phase, players communicate with each other. At the end of this phase, each player thus knows the statistics of all players on all arms, so that this decentralized problem becomes similar to the centralized one. 
After alternating enough times between exploration and communication phases, sub-optimal arms are eliminated and players are fixed to different optimal arms and will exploit them until stage $T$. The complete pseudocode of \algoone is given in Algorithm~\ref{algo:comm}, Appendix~\ref{app:alg1}.



\subsubsection{Initialization phase}
\label{Phase init}

The objective of the first phase is to estimate the number of players $M$ and to assign \textbf{internal ranks} to players. First, players follow the Musical Chairs algorithm \cite{musicalchair}, described in Pseudocode~\ref{algo:MC}, Appendix~\ref{app:alg1}, during $T_0 \coloneqq \lceil K \log(T) \rceil$ steps in order to reach an \textbf{orthogonal setting}, \ie a position where they are all pulling different arms. The index of the arm pulled by a player at stage $T_0$ will then be her \textbf{external rank}.

The second procedure, given by Pseudocode~\ref{algo1:estimM} in Appendix~\ref{app:alg1}, determines $M$ and assigns a unique internal rank in $[M]$ to each player. For example, if there are three players on arms $5$, $7$ and $2$ at $t=T_0$, their external ranks are $5$, $7$ and $2$ respectively, while their internal ranks are $2$, $3$ and $1$. 
Roughly speaking, the players follow each other sequentially hopping through all the arms so that players with external ranks $k$ and $k'$ collide exactly after a time $k+k'$. Each player then deduces $M$ and her internal rank from observed collisions during this procedure that lasts $2K$ steps.

In the next phases, active players will always know the set of active players $[M_p]$. This is how the initial symmetry among players is broken and it allows the decentralized algorithm to establish communication protocols.

\subsubsection{Exploration phase}
\label{Phase explo}

During the $p$-th exploration phase, active players sequentially hop among the active arms for $K_p 2^p$ steps. Any active arm is thus pulled $2^p$ times by each active player. Using their internal rank, players start and remain in an orthogonal setting during the exploration phase, which is collision-free. 

We denote by $B_s = 3\sqrt{\frac{\log(T)}{2s}}$ the error bound after $s$ pulls and by $T_k(p)$ (resp. $S_k(p)$) the centralized number of pulls (resp. sum of rewards) for the arm $k$ during the $p$ first exploration phases, \ie $T_k(p)=\sum_{j=1}^M T_k^j(p)$ where $T_k^j(p)$ is the number of pulls for the arm $k$ by player $j$ during the $p$ first exploration phases. During the communication phase, quantized rewards $\widetilde{S}_k^j(p)$ will be communicated between active players as described in Section~\ref{Phase comm}.

After a succession of two phases (exploration and communication), an arm $k$ is \textbf{accepted} if 
\begin{small}
\begin{equation*}
\# \Big\lbrace i \in [K_p]\, \big|\, \widetilde{\mu}_{k}(p) - B_{T_k(p)} \geq \widetilde{\mu}_i(p) + B_{T_i(p)} \Big\rbrace \geq K_p-M_p,\end{equation*}
\end{small}where $\widetilde{\mu}_k(p)=\frac{\sum_{m=1}^M \widetilde{S}_k^j(p)}{T_k(p)}$ is the centralized quantized empirical mean of the arm $k$\footnote{For a player $j$ already exploiting since the $p^j$-th phase, we instead use the last statistic $\widetilde{S}_k^j(p) = \widetilde{S}_k^j(p^j)$.}, which is an approximation of $\hat{\mu}_k(p)=\frac{S_k(p)}{T_k(p)}$. This inequality implies that $k$ is among the top-$M_p$ active arms with high probability. In the same way, $k$ is \textbf{rejected} if 
\begin{small}
\begin{equation*}
\#\Big\lbrace i \in [K_p]\, \big|\, \widetilde{\mu}_{i}(p) - B_{T_i(p)} \geq \widetilde{\mu}_k(p) + B_{T_k(p)} \Big\rbrace \geq M_p,\end{equation*}
\end{small}meaning that there are at least $M_p$ active arms better than $k$ with high probability.
Notice that each player $j$ uses her own quantized statistics $\widetilde{S}_k^j(p)$ to accept/reject an arm instead of the exact ones $S_k^j(p)$. Otherwise, the estimations $\widetilde{\mu}_k(p)$ would indeed differ between the players as well as the sets of accepted and rejected arms. With Bernoulli distributions, the quantization becomes unnecessary and the confidence bound can be chosen as $B_s = \sqrt{2 \log(T)/s}$.

\subsubsection{Communication phase}
\label{Phase comm}

In this phase, each active player communicates, one at a time, her statistics of the active arms to all other active players. Each player has her own communicating arm, corresponding to her internal rank. When the player $j$ is communicating, she sends a bit at a time step to the player $l$ by deciding which arm to pull: a $1$ bit is sent by pulling the communicating arm of player $l$ (a collision occurs) and a $0$ bit by pulling her own arm. The main originality of \algoone comes from this trick which allows implicit communication through collisions and is used in subsequent papers \cite{bubeck2019, elimetc, proutiere2019}. In an independent work, \citet{tibrewal2019} also proposed an algorithm using similar communication protocols for the heterogeneous case.

As an arm is pulled $2^n$ times by a single player during the $n$-th exploration phase, it has been pulled $2^{p+1}-1$ times in total at the end of the $p$-th phase and the statistic $S_k^j(p)$ is a real number in $[0, 2^{p+1}-1]$. Players then send a quantized \textbf{integer} statistic $\widetilde{S}_k^j(p) \in [2^{p+1}-1]$ to each other in $p+1$ bits, \ie collisions. Let $n = \lfloor S_k^j(p) \rfloor$ and $d = S_k^j(p) - n$ be the integer and decimal parts of $S_k^j(p)$, the quantized statistic is then $n+1$ with probability $d$ and $n$ otherwise, so that $\mathbb{E}[\widetilde{S}_k^j(p)] = S_k^j(p)$.

An active player can have three possible statuses during the communication phase:
\vspace{-0.5em}
\begin{enumerate}
\item either she is receiving some other players' statistics about the arm $k$. In that case, she proceeds to \receive (see Pseudocode~\ref{pseudo:receive}).
\item Or she is sending her quantized statistics about arm $k$ to player $l$ (who is then receiving). In that case, she proceeds to \send (see Pseudocode~\ref{pseudo:send}) to send them in a time $p+1$.
\item Or she is pulling her communicating arm, while waiting for other players to finish communicating statistics among them.
\end{enumerate}

\begin{figure*}[h]\begin{minipage}{0.45\textwidth}
   \centering
   \begin{pseudocode}[H]
   \vspace{-1,0cm}
     \begin{algorithm}[H]\small
		\caption*{\receive}
		\hspace*{\algorithmicindent}  \textbf{Input:}  $p$ (phase number), $l$ (own internal rank), $[K_p]$ (set of active arms) \\
		\hspace*{\algorithmicindent} \textbf{Output:}  $s$ (statistic sent by the sending player)
		\begin{algorithmic}[1]
		\STATE $s \gets 0$ and $\pi \gets$ index of the $l$-th active arm
		\FOR{$n = 0, \ldots , p$ }
		\STATE Pull $\pi$
		\STATE \algorithmicif\ $\eta_{\pi} (t) = 1$ \algorithmicthen\ \COMMENT{other player sends $1$}
		\STATE $s \gets s + 2^n$ \algorithmicendif
		\ENDFOR
		\RETURN $s$ \COMMENT{sent statistics}
		\end{algorithmic}
		\end{algorithm}
		\setlength{\floatsep}{0.1cm}
		\vspace{-0,5cm}
		\caption{\label{pseudo:receive}receive statistics of length \\$p+1$.}
		\end{pseudocode}
   \end{minipage} \hfill   
   \begin{minipage}{0.45\textwidth}
     \centering
          \begin{pseudocode}[H]
          \vspace{-1,0cm}
          \begin{algorithm}[H]\small
		\caption*{\send}
		\hspace*{\algorithmicindent} \textbf{Input:} $l$ (player receiving), $s$ (statistics to send), $p$ (phase number), $j$ (own internal rank), $[K_p]$ (set of active arms)
		\begin{algorithmic}[1]
		\STATE $\mathbf{m} \gets$ binary writing of $s$ of length $p+1$, \ie $s = \sum_{n=0}^p m_n 2^n$
		\FOR{$n = 0, \ldots , p$}
		\STATE \algorithmicif\ $m_n = 1$\ \algorithmicthen\ 
		\STATE \hspace*{\algorithmicindent} Pull the $l$-th active arm \COMMENT{send $1$}
		\STATE \algorithmicelse\ Pull the $j$-th active arm\ \COMMENT{send $0$}
		\STATE \algorithmicendif
		\ENDFOR
		\end{algorithmic}
		\end{algorithm}
		\vspace{-0,5cm}
		\caption{\label{pseudo:send}send statistics $s$ of length \\ $p+1$ to player $l$.}
		\end{pseudocode}
   \end{minipage}
   \end{figure*}
      
Communicated statistics are all of length $p+1$, even if they could be sent with shorter messages, in order to maintain synchronization among players. 
Using their internal ranks, the players can communicate in turn without interfering with each other.
The general protocol for each communication phase is described in Pseudocode~\ref{pseudo:comm} below. 

\begin{pseudocode}[h]
          \vspace{-0,5cm}
          \begin{algorithm}[H]\small
		\caption*{\textbf{Communication Protocol}}
		\hspace*{\algorithmicindent} \textbf{Input:} $\mathbf{s}$ (personal statistics of previous phases), $p$ (phase number), $j$ (own internal rank), $[K_p]$ (set of active arms), $[M_p]$ (set of active players) \\
		\hspace*{\algorithmicindent} \textbf{Output:} $\mathbf{\widetilde{S}}$ (quantized statistics of all active players)
		\begin{algorithmic}[1]
				\STATE For all $k$, sample $\widetilde{s}[k] = \begin{cases} \lfloor s[k] \rfloor +1 \text{ with probability } s[k] - \lfloor s[k] \rfloor \\ \lfloor s[k] \rfloor \text{ otherwise} \end{cases} $ \COMMENT{quantization}
		\STATE Define $E_p \coloneqq \{ (i,l,k) \in [M_p]\times [M_p]\times [K_p] \ | \ i \neq l \}$ and set $\mathbf{\widetilde{S}^j} \gets \mathbf{\widetilde{s}}$

\STATE \algorithmicfor\ $(i,l,k) \in E_p$\ \algorithmicdo  \COMMENT{Player $i$ sends stats of arm $k$ to player $l$}\label{alg1:comm0}
\STATE \hspace*{\algorithmicindent} \algorithmicif\ $i=j$\ \algorithmicthen\ Send $(l, \widetilde{s}[k], p, j, [K_p])$ \COMMENT{player communicating}
\STATE \hspace*{\algorithmicindent}  \algorithmicelsif\ $l=j$ \algorithmicthen\ $\widetilde{S}^i[k] \gets$ Receive$(p, j, [K_p])$ \COMMENT{player receiving}
\STATE \hspace*{\algorithmicindent}  \algorithmicelse\ \algorithmicfor\ $p+1$ time steps\ \algorithmicdo\ Pull the $j$-th active arm\ \algorithmicendfor \COMMENT{wait while others communicate}
\STATE \hspace*{\algorithmicindent} \algorithmicendif
\STATE \algorithmicendfor
\RETURN $\mathbf{\widetilde{S}}$
		\end{algorithmic}
		\end{algorithm}
		\vspace{-0,5cm}
		\caption{\label{pseudo:comm}player with rank $j$ proceeds to the $p$-th communication phase.}
				\vspace{-1em}
		\end{pseudocode}
		
At the end of the communication phase, all active players know the statistics $\widetilde{S}_k^j(p)$ and so which arms to accept or reject. Rejected arms are removed right away from the set of active arms. Thanks to the assigned ranks, accepted arms are assigned to one player each. The remaining active players then update both sets of active players and arms as described in Algorithm~\ref{algo:comm}, line~\ref{alg1:update0}.

This communication protocol uses the fact that a bit can be sent with a single collision. Without sensing, this can not be done in a single time step, but  communication is still somehow possible. A bit can then be sent in $\frac{\log(T)}{\mu_{(K)}}$ steps with probability $1 - \frac{1}{T}$. Using this trick, two different algorithms relying on communication protocols are proposed in Appendix~\ref{sec:nosens1} for the No Sensing setting.

\subsubsection{Regret bound of \algoone}

Theorem~\ref{thm:upperbound} bounds the expected regret incurred by \algoone[.] Due to space constraints, its proof is delayed to Appendix~\ref{app:regretalgo1}.

\begin{thm}
\label{thm:upperbound}
With the choice $T_0 = \lceil K \log(T) \rceil$, for any given set of parameters $K$, $M$ and $\pmb{\mu}$: 
\begin{small}
\begin{align*}
\mathbb{E}\big[ R_T \big] \leq & \ c_1 \!\! {\mathlarger \sum_{k > M}} \min \bigg\lbrace \frac{\log(T)}{\mu_{(M)} - \mu_{(k)}}, \sqrt{T \log(T)} \bigg\rbrace + c_2 KM \log(T) \\ & + c_3 KM^3 \log^2 \left( \min \bigg\lbrace \frac{\log(T)}{(\mu_{(M)}-\mu_{(M+1)})^2}, T \bigg\rbrace \right) \end{align*}
\end{small}
where $c_1$, $c_2$ and $c_3$ are universal constants.
\end{thm}

The first, second and third terms respectively correspond to the regret incurred by the exploration, initialization and communication phases, which dominate the regret due to low probability events of bad initialization or incorrect estimations. Notice that the minmax regret scales with $\mathcal{O}(K\sqrt{T\log(T)})$.\\
Experiments on synthetic data are described in Appendix~\ref{sec:expe}. They empirically confirm that \algoone scales better than \textsc{mct}op\textsc{m} \citep{besson} with the gaps $\Delta$, besides having a smaller minmax regret.

\subsection{In contradiction with existing lower bounds?}

Theorem~\ref{thm:upperbound} is in contradiction with the two existing lower bounds \citep{besson, lower2}, however \algoone respects the conditions required for both. It was thought that the decentralized lower bound was $\Omega \left( M \sum_{k > M} \frac{\log(T)}{\mu_{(M)} - \mu_{(k)}}\right)$, while the centralized lower bound was already known to be $\Omega \left( \sum_{k > M} \frac{\log(T)}{\mu_{(M)} - \mu_{(k)}}\right)$ \citep{anantharam}. However, it appears that the asymptotic regret of the decentralized case is not that much different from the latter, at least if players are synchronized. Indeed, \algoone takes advantage of this synchronization to establish communication protocols as players are able to communicate through collisions. Subsequent papers \citep{elimetc, proutiere2019} recently improved the communication protocols of \algoone to obtain both initialization and communication costs constant in $T$, confirming that the lower bound of the centralized case is also tight for the decentralized model considered so far.

\citet{lower2} proved the lower bound ``by considering the best case that they do not collide''. This is only true if colliding does not provide valuable information and the policies just maximize the losses at each round, disregarding the information gathered for the future. Our algorithm is built upon the idea that the value of the information provided by collisions can exceed  in the long run the immediate loss in rewards  (which is standard in dynamic programming or reinforcement learning for instance).
The mistake of \citet{besson} is found in the proof of Lemma 12 after the sentence ``We now show that second term in (25) is zero''. The conditional expectation cannot be put inside/outside of the expectation as written and the considered term, which corresponds to the difference of information given by collisions for two different distributions, is therefore not zero. \\
These two lower bounds disregarded  the amount of information that can be deduced from collisions, while \algoone obviously takes advantage of this information. 

Our exploration regret reaches, up to a constant factor, the lower bound of the centralized problem \citep{anantharam}. 
Although it is sub-logarithmic in time, the communication cost scales with $KM^3$ and can thus be predominant in practice. Indeed for large networks, $M^3$ can easily be greater than $\log(T)$ and the communication cost would then prevail over the other terms. This highlights the importance of the parameter $M$ in multiplayer MAB and future work should focus on the dependency in both $M$ and $T$ instead of only considering asymptotic results in $T$.

Synchronization is not a reasonable assumption for practical purposes and  it also leads to undesirable algorithms relying on communication protocols such as \algoone[.]We thus claim that this assumption should be removed in the multiplayer MAB and the \textit{dynamic model} should be considered instead. However, this problem seems complex to model formally. Indeed, if players stay in the game only for a very short period, learning is not possible. The difficulty to formalize an interesting and nontrivial dynamic model may explain why most of the literature focused on the static model so far.

%
\vspace{-0.5em}

\section{Without synchronization, the dynamic setting}
\label{sec:nosens2}
\vspace{-0.5em}

From now on, we no longer assume that players can communicate using synchronization. In the previous section, it was crucial that all exploration/communication phases start and end at the same time. This assumption is clearly unrealistic and should be alleviated, as radios do not start and end transmitting simultaneously. We also consider the more difficult No Sensing setting in this section.

We assume in the following that players do not leave the game once they have started. Yet, we mention that our results can also be adapted to the cases when players can leave the game during specific intervals or share an internal synchronized clock \citep{musicalchair}. If the time is divided in several intervals, \algothree can be run independently on each of these intervals as suggested by \citet{musicalchair}. In some cases, players will be leaving in the middle of these intervals, leading to a large regret. But for any other interval, every player stays until its end, thus satisfying Assumption~\ref{hyp:asynch}.

In this section, Assumption~\ref{hyp:asynch} holds. At each stage $t = t_j + \tau_j$, player $j$ does not know $t$ but only $t_j$ (duration since joining). We denote by $T^j = T - \tau_j$ the (known) time horizon of player $j$.
\vspace{-0.5em}
\subsection{A logarithmic regret algorithm}

As synchronization no longer holds,  we propose the \algothree algorithm, relying on different tools than \algoone[.]The main ideas of \algothree are given in Section~\ref{sec:intuitions}. Its thorough description as well as the proof of the regret bound are delayed to Appendix~\ref{app:alg30} due to space constraints.

The regret incurred by \algothree in the dynamic No Sensing model is given by Theorem~\ref{thm:dynamicregret} and its proof is delayed to Appendix~\ref{app:alg3proofs}.
We also mention that \algothree leads to a Pareto optimal configuration in the more general problem where users' reward distributions differ \citep{diffmeans0, differentmeans, diffmeans2, GoT2018}.
\begin{thm}
\label{thm:dynamicregret}
In the dynamic setting, the regret incurred by \algothree is upper bounded as follows:
\begin{equation*}
\mathbb{E}[R_T] \leq \mathcal{O}\left( \frac{M^2K \log(T)}{\mu_{(M)}} + \frac{MK \log(T)}{\bar{\Delta}_{(M)}^2}  \right),
\end{equation*}
where $M = \#\mathbf{M}(T)$ is the total number of players in the game and $\bar{\Delta}_{(M)} = \min\limits_{i=1,..., M}(\mu_{(i)} - \mu_{(i+1)}) $.
\end{thm}

\subsection{A communication-less protocol}
\label{sec:intuitions}

\algothree['s]ideas are easy to understand but the upper bound proof is quite technical. This section gives some intuitions about \algothree and its performance guarantees stated in Theorem~\ref{thm:dynamicregret}.

A player will only follow two different sampling strategies:
either she samples uniformly at random in $[K]$ during the exploration phase; or she exploits an arm and pulls it until the final horizon. In the first case, the exploration of the other players is not  too disturbed by collisions as they only change the mean reward of all arms by a common multiplicative term. In the second case, the exploited arm will appear as sub-optimal to the other players, which is actually convenient for them as this arm is now exploited.  

During the exploration phase, a player  will update a set of arms called  \texttt{Occupied} $\subset [K]$ and an ordered list of arms called \texttt{Preferences} $\in [K]^{\star}$.  As soon as an arm is detected as occupied (by another player), it is then added to  \texttt{Occupied} (which is the empty set at the beginning). If an arm is discovered to be the best one amongst those that are neither in \texttt{Occupied}  nor in \texttt{Preferences}, it is then  added to \texttt{Preferences} (at the last position). An arm is \textbf{active} for player $j$ if it was neither added to \texttt{Occupied} nor to \texttt{Preferences} by this player yet.

To handle the fact that players can enter the game at anytime, we introduce the quantity $\gamma^j(t)$,  the expected multiplicative factor of the means defined by
\begin{small}
\begin{equation*}
\gamma^j(t) = \frac{1}{t} \sum_{t'= 1 + \tau_j}^{t + \tau_j} \mathbb{E}\Big[(1-\frac{1}{K})^{m_{t'} - 1} \Big],
\end{equation*}
\end{small}where $m_t$ is the number of players in their exploration phase at time $t$. The value of $\gamma^j(t)$ is unknown to the player and random but it only affects the analysis of \algothree[]and not how it runs.

The objective of the algorithm is still to form  estimates and confidence intervals of the performances of  arms. However, it might happen that the true mean $\mu_k$ does not belong to this confidence interval. Indeed, this is only true for $\gamma^j(t) \mu_k$, if the arm $k$ is still free (not exploited). This is the first point of Lemma~\ref{lemma:nosensoccupied} below. Notice that as soon as the confidence interval for the arm $i$ dominates the confidence interval for the arm $k$, then it must hold that $\gamma^j(t) \mu_i \geq \gamma^j(t) \mu_k$ and thus arm $i$ is better than $k$. 

The second crucial point is to detect when an arm $k$ is exploited by another player. This detection will happen if a player receives too many 0 rewards successively (so that it is statistically very unlikely that this arm is not occupied). The number of zero rewards needed for player $j$ to disregard arm $k$ is denoted by  $L_k^j$, which is sequentially updated during the process (following the rule of Equation~\eqref{eq:updateTf} in Appendix~\ref{app:alg3}), so that $L_k^j \geq 2e \log(T^j)/\mu_k$. As the probability of observing a $0$ reward on a free arm $k$ is smaller than $1-\mu_k/e$, no matter the current number of players,  observing $L_k^j$ successive $0$ rewards on an unexploited arm happens with probability smaller than $ \frac{1}{(T^j)^2}$. 

The second point of Lemma~\ref{lemma:nosensoccupied} then states that an exploited arm will either be quickly detected as occupied after observing $L_k^j$ zeros (if $L_k^j$ is small enough) or  its average reward will quickly drop because it now gives zero rewards (and it will be dominated by another arm after a relatively small number of pulls). The proof of Lemma~\ref{lemma:nosensoccupied} is delayed to Appendix~\ref{app:alg3proofs}.

\begin{lemm}\label{lemma:nosensoccupied} 
We denote by $\hat{r}_k^j(t)$ the empirical average reward of arm $k$ for player $j$ at stage $t+ \tau_j$.
\begin{enumerate}
\item For any player $j$ and arm $k$, if $k$ is still free at stage $t+\tau_j$, then 
$$\mathbb{P} \Big[|\hat{r}_k^j(t) - \gamma^j(t) \mu_k| > 2\sqrt\frac{6 \ K \log(T^j)}{t} \Big] \leq \frac{4}{(T^j)^2}.$$ 
We then say that the arm $k$ is \textbf{correctly estimated} by player $j$ if $|\hat{r}_k^j(t) - \gamma^j(t) \mu_k| \leq 2\sqrt\frac{6 \ K \log(T^j)}{t}$ holds as long as $k$ is free.
 \item On the other hand, if $k$ is exploited by some player $j' \neq j$ at stage $t^0+\tau_j$, then, conditionally on the correct estimation of all the arms by player $j$, with probability $1-\mathcal{O}\left(\frac{1}{T^j}\right)$:
 \vspace{-1em}
\begin{itemize}
\item either $k$ is added to \texttt{Occupied} at a stage at most $t^0 + \tau_j + \mathcal{O} \left( \frac{K \log(T)}{\mu_k}\right)$ by player $j$,
\item or   $k$ is dominated by another unoccupied arm $i$ (for player $j$) at stage at most $\mathcal{O} \left( \frac{K \log(T)}{\mu_i^2} \right) + \tau_j$.\end{itemize}
\end{enumerate}
\end{lemm}
%

It remains to describe how players start exploiting arms. After some time (upper-bounded by Lemma~\ref{lemma:nosens2} in Appendix~\ref{app:alg3proofs}), an arm which is still free and such that all  better arms are occupied will be detected as the best remaining one. The player will try to occupy it, and this  happens as soon as she gets a positive reward from it: either she succeeds and starts exploiting it, or she fails and assumes it  is occupied by another player (this only takes a few number of steps, see Lemma~\ref{lemma:nosensoccupied}). In the latter case, she resumes exploring until she detects the next available best arm. With high probability, the player will necessarily end up exploiting an arm while all the better arms are already exploited by other players.
%
\vspace{-0.3cm}

\section{Conclusion}
\vspace{-0.2cm}

We have  presented algorithms for different multiplayer bandits models. The first one illustrates why the assumption of synchronization between the players is basically equivalent to allowing communication. Since communication through collisions is possible with other players at a sub-logarithmic cost, the decentralized multiplayer bandits is almost equivalent to the centralized one for the considered model. However, this communication cost has a large dependency in the number of agents in the network. Future work should then focus on considering both the dependency in time and the number of players as well as developing efficient communication protocols.

Our major claim is that synchronization should not be considered anymore, unless communication is allowed. We thus introduced a dynamic model  and proposed the first algorithm with a logarithmic regret. 

\subsubsection*{Acknowledgments}
This work was supported in part by a public grant as part of the Investissement d'avenir project, reference ANR-11-LABX-0056-LMH, LabEx LMH, in a joint call with Gaspard Monge Program for optimization, operations research and their interactions with data sciences.

\addcontentsline{toc}{section}{References}
\bibliographystyle{plainnat}
\bibliography{bibliography}

\begin{thebibliography}{28}
\providecommand{\natexlab}[1]{#1}
\providecommand{\url}[1]{\texttt{#1}}
\expandafter\ifx\csname urlstyle\endcsname\relax
  \providecommand{\doi}[1]{doi: #1}\else
  \providecommand{\doi}{doi: \begingroup \urlstyle{rm}\Url}\fi

\bibitem[Agrawal(1995)]{agrawal95}
R.~Agrawal.
\newblock Sample mean based index policies with o(log n) regret for the
  multi-armed bandit problem.
\newblock \emph{Advances in Applied Probability}, 27\penalty0 (4):\penalty0
  1054--1078, 1995.

\bibitem[Anandkumar et~al.(2011)Anandkumar, Michael, Tang, and
  Swami]{anandkumar}
A.~Anandkumar, N.~Michael, A.~K. Tang, and A.~Swami.
\newblock Distributed algorithms for learning and cognitive medium access with
  logarithmic regret.
\newblock \emph{IEEE Journal on Selected Areas in Communications}, 29\penalty0
  (4):\penalty0 731--745, 2011.

\bibitem[Anantharam et~al.(1987)Anantharam, Varaiya, and Walrand]{anantharam}
V.~Anantharam, P.~Varaiya, and J.~Walrand.
\newblock Asymptotically efficient allocation rules for the multiarmed bandit
  problem with multiple plays-part i: I.i.d. rewards.
\newblock \emph{IEEE Transactions on Automatic Control}, 32\penalty0
  (11):\penalty0 968--976, 1987.

\bibitem[Auer et~al.(2002)Auer, Cesa-Bianchi, and Fischer]{Auer2002}
P.~Auer, N.~Cesa-Bianchi, and P.~Fischer.
\newblock Finite-time analysis of the multiarmed bandit problem.
\newblock \emph{Machine learning}, 47\penalty0 (2-3):\penalty0 235--256, 2002.

\bibitem[Avner and Mannor(2014)]{mega}
O.~Avner and S.~Mannor.
\newblock Concurrent bandits and cognitive radio networks.
\newblock In \emph{Joint European Conference on Machine Learning and Knowledge
  Discovery in Databases}, pages 66--81. Springer, 2014.

\bibitem[Avner and Mannor(2015)]{differentmeans}
O.~Avner and S.~Mannor.
\newblock Learning to coordinate without communication in multi-user
  multi-armed bandit problems.
\newblock \emph{arXiv preprint arXiv:1504.08167}, 2015.

\bibitem[Avner and Mannor(2018)]{diffmeans2}
O.~Avner and S.~Mannor.
\newblock Multi-user communication networks: A coordinated multi-armed bandit
  approach.
\newblock \emph{arXiv preprint arXiv:1808.04875}, 2018.

\bibitem[Besson and Kaufmann(2018)]{besson}
L.~Besson and E.~Kaufmann.
\newblock {Multi-Player Bandits Revisited}.
\newblock In \emph{{Algorithmic Learning Theory}}, Lanzarote, Spain, 2018.

\bibitem[Bistritz and Leshem(2018)]{GoT2018}
I.~Bistritz and A.~Leshem.
\newblock Distributed multi-player bandits-a game of thrones approach.
\newblock In \emph{Advances in Neural Information Processing Systems}, pages
  7222--7232. 2018.

\bibitem[Boursier et~al.(2019)Boursier, Kaufmann, Mehrabian, and
  Perchet]{elimetc}
E.~Boursier, E.~Kaufmann, A.~Mehrabian, and V.~Perchet.
\newblock A practical algorithm for multiplayer bandits when arm means vary
  among players.
\newblock \emph{arXiv preprint arXiv:1902.01239}, 2019.

\bibitem[Bubeck and Cesa{-}Bianchi(2012)]{survey}
S.~Bubeck and N.~Cesa{-}Bianchi.
\newblock Regret analysis of stochastic and nonstochastic multi-armed bandit
  problems.
\newblock \emph{Foundations and Trends{\textregistered} in Machine Learning},
  5\penalty0 (1):\penalty0 1--122, 2012.

\bibitem[Bubeck et~al.(2013)Bubeck, Wang, and Viswanathan]{SAR}
S.~Bubeck, T.~Wang, and N.~Viswanathan.
\newblock Multiple identifications in multi-armed bandits.
\newblock In \emph{International Conference on Machine Learning}, pages
  258--265, 2013.

\bibitem[Bubeck et~al.(2019)Bubeck, Li, Peres, and Sellke]{bubeck2019}
S.~Bubeck, Y.~Li, Y.~Peres, and M.~Sellke.
\newblock Non-stochastic multi-player multi-armed bandits: Optimal rate with
  collision information, sublinear without.
\newblock \emph{arXiv preprint arXiv:1904.12233}, 2019.

\bibitem[Degenne and Perchet(2016)]{degenne2016anytime}
R.~Degenne and V.~Perchet.
\newblock Anytime optimal algorithms in stochastic multi-armed bandits.
\newblock In \emph{International Conference on Machine Learning}, pages
  1587--1595, 2016.

\bibitem[Joshi et~al.(2018)Joshi, Kumar, Yadav, and Darak]{sequentialhopping}
H.~Joshi, R.~Kumar, A.~Yadav, and S.~J. Darak.
\newblock Distributed algorithm for dynamic spectrum access in
  infrastructure-less cognitive radio network.
\newblock In \emph{2018 IEEE Wireless Communications and Networking Conference
  (WCNC)}, pages 1--6, 2018.

\bibitem[Jouini et~al.(2009)Jouini, Ernst, Moy, and Palicot]{jouini}
W.~Jouini, D.~Ernst, C.~Moy, and J.~Palicot.
\newblock Multi-armed bandit based policies for cognitive radio's decision
  making issues.
\newblock In \emph{2009 3rd International Conference on Signals, Circuits and
  Systems (SCS)}, 2009.

\bibitem[Kalathil et~al.(2014)Kalathil, Nayyar, and Jain]{diffmeans0}
D.~Kalathil, N.~Nayyar, and R.~Jain.
\newblock Decentralized learning for multiplayer multiarmed bandits.
\newblock \emph{IEEE Transactions on Information Theory}, 60\penalty0
  (4):\penalty0 2331--2345, 2014.

\bibitem[Komiyama et~al.(2015)Komiyama, Honda, and Nakagawa]{centralized2}
J.~Komiyama, J.~Honda, and H.~Nakagawa.
\newblock Optimal regret analysis of thompson sampling in stochastic
  multi-armed bandit problem with multiple plays.
\newblock In \emph{International Conference on Machine Learning}, pages
  1152--1161, 2015.

\bibitem[Lai and Robbins(1985)]{LaiRobbins}
T.~L. Lai and H.~Robbins.
\newblock Asymptotically efficient adaptive allocation rules.
\newblock \emph{Advances in applied mathematics}, 6\penalty0 (1):\penalty0
  4--22, 1985.

\bibitem[Liu and Zhao(2010)]{lower2}
K.~Liu and Q.~Zhao.
\newblock Distributed learning in multi-armed bandit with multiple players.
\newblock \emph{IEEE Transactions on Signal Processing}, 58\penalty0
  (11):\penalty0 5667--5681, 2010.

\bibitem[Lugosi and Mehrabian(2018)]{gabor}
G.~Lugosi and A.~Mehrabian.
\newblock Multiplayer bandits without observing collision information.
\newblock \emph{arXiv preprint arXiv:1808.08416}, 2018.

\bibitem[Perchet and Rigollet(2013)]{perchet2013}
V.~Perchet and P.~Rigollet.
\newblock The multi-armed bandit problem with covariates.
\newblock \emph{The Annals of Statistics}, 41\penalty0 (2):\penalty0 693--721,
  2013.

\bibitem[Perchet et~al.(2015)Perchet, Rigollet, Chassang, and
  Snowberg]{batchedbandits}
V.~Perchet, P.~Rigollet, S.~Chassang, and E.~Snowberg.
\newblock Batched bandit problems.
\newblock In \emph{Proceedings of The 28th Conference on Learning Theory},
  pages 1456--1456, 2015.

\bibitem[Proutiere and Wang(2019)]{proutiere2019}
A.~Proutiere and P.~Wang.
\newblock An optimal algorithm in multiplayer multi-armed bandits, 2019.

\bibitem[Robbins(1952)]{robbins52}
H.~Robbins.
\newblock Some aspects of the sequential design of experiments.
\newblock \emph{Bulletin of the American Mathematical Society}, 58\penalty0
  (5):\penalty0 527--535, 1952.

\bibitem[Rosenski et~al.(2016)Rosenski, Shamir, and Szlak]{musicalchair}
J.~Rosenski, O.~Shamir, and L.~Szlak.
\newblock Multi-player bandits--a musical chairs approach.
\newblock In \emph{International Conference on Machine Learning}, pages
  155--163, 2016.

\bibitem[Thompson(1933)]{thompson33}
W.~R. Thompson.
\newblock On the likelihood that one unknown probability exceeds another in
  view of the evidence of two samples.
\newblock \emph{Biometrika}, 25\penalty0 (3-4):\penalty0 285--294, 1933.

\bibitem[Tibrewal et~al.(2019)Tibrewal, Patchala, Hanawal, and
  Darak]{tibrewal2019}
H.~Tibrewal, S.~Patchala, M.K. Hanawal, and S.J. Darak.
\newblock Distributed learning and optimal assignment in multiplayer
  heterogeneous networks.
\newblock In \emph{IEEE INFOCOM}, pages 1693--1701, 2019.

\end{thebibliography}

\newpage
\appendix
%
\section{Complementary material for Section~\ref{sec:synchcomm}}
\subsection{Algorithm description}
\label{app:alg1}

We here describe in detail the \algoone algorithm. All the pseudocodes are described from the point of view of a single player, which is the natural way to describe a decentralized algorithm. First, this algorithm relies on the Musical Chairs algorithm, introduced by \citet{musicalchair}. We recall it in Pseudocode~\ref{algo:MC}.

\begin{pseudocode}[H]
\vspace{-0,5cm}
\begin{algorithm}[H]
\caption*{\textbf{\musicalchair Protocol}}
\hspace*{\algorithmicindent} \textbf{Input:}  $[K_p]$ (set of active arms), $T_0$ (time of procedure) \\
\hspace*{\algorithmicindent} \textbf{Output:}  \texttt{Fixed} (external rank)
\begin{algorithmic}[1]
\STATE Initialize \texttt{Fixed} $\gets -1$
\FOR{$T_0$ time steps}
\STATE \algorithmicif\ \texttt{Fixed} $=-1$ \algorithmicthen
\STATE \hspace*{\algorithmicindent} Sample $k$ uniformly at random in $[K_p]$ and play it in round $t$
\STATE \hspace*{\algorithmicindent} \algorithmicif\ $\eta_k(t)=0$ ($r_k(t)>0$ for No Sensing setting) \algorithmicthen\
\STATE \hspace*{2\algorithmicindent}\texttt{Fixed} $\gets k$\ \algorithmicendif \COMMENT{The player stays in arm $k$ if no collision}

\STATE \algorithmicelse\ Play \texttt{Fixed} \algorithmicendif
\ENDFOR
\RETURN \texttt{Fixed} \hfill \COMMENT{External rank}
\end{algorithmic}
\end{algorithm}
\vspace{-0,5cm}
\caption{\label{algo:MC} reach an orthogonal setting in $T_0$ steps.}
\vspace{-0,3cm}
\end{pseudocode}

The initialization phase then consists of a second procedure. Its purpose is to estimate $M$ and to assign different ranks in $[M]$ to all players. This procedure is described in Pseudocode~\ref{algo1:estimM} below. \algoone is finally described in Algorithm~\ref{algo:comm}.


\begin{pseudocode}[h!]
\vspace{-0,5cm}
\begin{algorithm}[H]
\caption*{\textbf{\estimm Protocol}}
\hspace*{\algorithmicindent} \textbf{Input:}  $k \in  [K]$ (external rank) \\
\hspace*{\algorithmicindent} \textbf{Output:}  $M$ (estimated number of players), $j$ (internal rank)
\begin{algorithmic}[1]
\STATE Initialize $M \gets 1$, $j \gets 1$ and $\pi \gets k$ \COMMENT{estimates of $M$ and the internal rank}
\FOR {$2k$ time steps}
\STATE Pull $\pi$;\qquad \algorithmicif\ $\eta_\pi(t) = 1$\ \algorithmicthen\ $M \gets M + 1$ and $j \gets j + 1$\ \algorithmicendif \COMMENT{increases if collision}
\ENDFOR
\FOR {$2(K-k)$ time steps}
\STATE $\pi \gets \pi + 1 \ (\textrm{mod } K)$ and pull $\pi$ \COMMENT{sequential hopping}
\STATE \algorithmicif\ $\eta_\pi (t)= 1$\ \algorithmicthen\ $M \gets M + 1$\ \algorithmicendif \COMMENT{increases if collision}
\ENDFOR
\RETURN $M, j$
\end{algorithmic}
\end{algorithm}
\vspace{-0,5cm}
\caption{\label{algo1:estimM} estimate $M$ and assign ranks to the players.}
\vspace{-0,3cm}
\end{pseudocode}

\begin{algorithm}[h]
\caption{\label{algo:comm} \algoone algorithm}
\hspace*{\algorithmicindent} \textbf{Input:} $T$ (horizon)

\begin{algorithmic}[1]
\STATE \textbf{Initialization Phase:}
\STATE Initialize $\texttt{Fixed} \gets -1$ and $T_0 \gets \lceil K \log(T) \rceil$
\STATE $k \gets $ \musicalchair($[K], T_0$) \label{alg1:init0}
\STATE $(M,j) \gets$ \estimm($k$) \label{alg1:init1} \COMMENT{estimated number of players and assigned internal rank}
\STATE Initialize $p \gets 1;\ M_p \gets M;\ [K_p] \gets [K]$ and $\mathbf{\widetilde{S}}, \mathbf{s}, \mathbf{T} \gets \textrm{Zeros}(K)$ \COMMENT{Zeros$(K)$ returns a\\ \hfill vector of  length $K$ containing only zeros}

\WHILE{\texttt{Fixed} $=-1$}
\vspace{0.2cm}

\STATE \textbf{Exploration Phase:}
\STATE $\pi \gets j$-th active arm \label{alg1:alter0} \COMMENT{start of a new phase}
\STATE  \algorithmicfor\ $K_p 2^p$ time steps \algorithmicdo \label{alg1:explo0}
\STATE \hspace*{\algorithmicindent}   $\pi \gets \pi +1 \ (\textrm{mod } [K_p])$ and play $\pi$ in round $t$ \COMMENT{sequential hopping}
\STATE \hspace*{\algorithmicindent}  $s[\pi] \gets s[\pi] + r_\pi (t)$ \COMMENT{Update individual statistics}
\STATE  \algorithmicendfor \label{alg1:explo1}
\vspace{0.2cm}

\STATE \textbf{Communication Phase:}
\STATE $\mathbf{\widetilde{S}_p} \gets $ Communication( $\mathbf{s}$, $p$, $j$, $[K_p]$, $[M_p]$) and $\mathbf{\widetilde{S}^l} \gets \mathbf{\widetilde{S}_p^l}$ for every active player $l$
\STATE $T[k] \gets T[k] + M_p 2^p$ for every active arm $k$
\STATE \textbf{Update Statistics:} \COMMENT{recall that $B_s = 3\sqrt{\frac{\log(T)}{2s}}$ here}
\STATE  Rej $\gets$ set of active arms $k$ verifying $\#\Big\lbrace i \in [K_p]\, \big|\, \frac{\sum\limits_{l=1}^{M} \widetilde{S}^l[i]}{T[i]} - B_{T[i]} \geq \frac{\sum\limits_{l=1}^{M} \widetilde{S}^l[k]}{T[k]} + B_{T[k]} \Big\rbrace \geq M_p$
\STATE  Acc $\gets$ set of active arms $k$ verifying $\# \Big\lbrace i \in [K_p]\, \big|\, \frac{\sum\limits_{l=1}^{M} \widetilde{S}^l[k]}{T[k]} - B_{T[k]} \geq \frac{\sum\limits_{l=1}^{M} \widetilde{S}^l[i]}{T[i]} + B_{T[i]} \Big\rbrace \geq K_p-M_p$, ordered according to their indices 
\STATE  \algorithmicif\ $M_p - j +1 \leq \texttt{length}(\textrm{Acc})$\ \algorithmicthen\ \texttt{Fixed} $\gets$ Acc$[M_p-j+1]$ \COMMENT{Start exploiting}
\STATE \algorithmicelse \COMMENT{Update all the statistics}
\STATE \hspace{\algorithmicindent} $M_p \gets M_p - \texttt{length(\textrm{Acc})}$ and $[K_p] \gets [K_p] \setminus (\textrm{Acc} \cup \textrm{Rej})$  \label{alg1:update0}
\STATE  \algorithmicendif \label{alg1:comm1}
\STATE   $p \gets p+1$

\ENDWHILE \label{alg1:alter1}
\vspace{0.2cm}

\STATE \textbf{Exploitation Phase:} Pull \texttt{Fixed} until $T$\label{alg1:explo}
\end{algorithmic}
\end{algorithm}

\newpage
\subsection{Regret analysis of \algoone}
\label{app:regretalgo1}

In this section, we prove the regret bound for \algoone algorithm given by Theorem~\ref{thm:upperbound}.
In what follows, the statement \textsl{``with probability $1 - \mathcal{O}(\delta(T))$, it holds that $f(T) = \mathcal{O}(g(T))$"} means that there is a universal constant $c \in \mathbb{R}_+$ such that $f(T) \leq c g(T)$ with probability at least $ 1 - c \delta(T)$.

We first decompose the regret as follows:
\begin{equation}
\label{eq:regdec}
R_T = R^{\text{init}} + R^{\text{comm}} + R^{\text{explo}},
\end{equation}
\begin{equation*}
\text{where } \left\{ \begin{split} \begin{aligned} & R^{\text{init}} = T_{\text{init}} {\mathlarger\sum_{k = 1}^M} \mu_{(k)} - \mathbb{E}_\mu \Big[{\mathlarger\sum_{t=1}^{T_{\text{init}}}} {\mathlarger \sum_{j = 1}^M} r^j(t) \Big]  
\text{ with } T_{\text{init}} = T_0 + 2K, \\
& R^{\text{comm}} = \mathbb{E}_\mu \Big[{\mathlarger\sum_{t \in \text{Comm}}}{\mathlarger \sum_{j=1}^M}  (\mu_{(j)} - r^j(t)) \Big] \text{ with Comm the set of communication steps,} \\
& R^{\text{explo}} = \mathbb{E}_\mu \Big[{\mathlarger\sum_{t \in \text{Explo}}}{\mathlarger \sum_{j=1}^M} (\mu_{(j)} - r^j(t)) \Big] 
\text{ with Explo} = \{T_{\text{init}} +1, \ldots, T \} \setminus \text{Comm.} \end{aligned} \end{split} \right.
\end{equation*}

A \textbf{communication step} is defined as a time step where a player is communicating statistics, \ie using \send[.]These terms respectively correspond to the regret due to the initialization phase, the communication and the regret of both exploration and exploitation phases.

\subsubsection{Initialization analysis}

The initialization regret is obviously bounded by $M (T_0 + 2K)$ as the initialization phase lasts $T_0 + 2K$ steps. Lemma~\ref{lemma:musicalchair} provides the probability to reach an orthogonal setting at time $T_0$. If this orthogonal setting is reached, the initialization phase is \textbf{successful}. In that case, the players then determine $M$ and a unique internal rank using Pseudocode~\ref{algo1:estimM}. This is shown by observing that players with external ranks $k$ and $k'$ will exactly collide at round $T_0 + k + k'$.

\begin{lemm}
\label{lemma:musicalchair}
After a time $T_0$, all players pull different arms with probability at least $ 1 - M \exp \left( - \frac{T_0}{K} \right)$.
\end{lemm}
\begin{proof}
\label{proof:musicalchair}
As there is at least one arm that is not played by all the other players at each time step, the probability of having no collision at time $t$ for a single player $j$ is lower bounded by $\frac{1}{K}$. 
It thus holds:
\begin{equation*}
\mathbb{P} \left[ \forall t \leq T_0, \eta^j(t) = 1 \right]  \leq \left( 1 - \frac{1}{K} \right) ^{T_0}  \leq \exp \left( -\frac{T_0}{K} \right).
\end{equation*}
For a single player $j$, her probability to encounter only collisions until time $T_0$ is at most $\exp \left( -\frac{T_0}{K} \right)$. The union bound over the $M$ players then yields the desired result.
\end{proof}

\subsubsection{Exploration regret}
\label{app:exploregret}

This section aims at proving Lemma~\ref{lemma:exploregretopt}, which bounds the exploration regret. 

\begin{lemm}
\label{lemma:exploregretopt}
With probability $1 - \mathcal{O} \left( \frac{K \log(T)}{T} + M \exp \left( - \frac{T_0}{K} \right) \right)$,

\begin{equation*}
R^{\mathrm{explo}} = \mathcal{O} \left( {\mathlarger\sum_{k >M}} \min \bigg\lbrace \frac{\log(T)}{\mu_{(M)} - \mu_{(k)}}, \sqrt{T \log(T)} \bigg\rbrace \right).
\end{equation*}
\end{lemm}

The proof of Lemma~\ref{lemma:exploregretopt} is divided in several auxiliary lemmas. It first relies on the correctness of the estimations before taking the decision to accept or reject any arm.

\begin{lemm}
\label{lemm:concentration} For any arm $k$ and positive integer $n$, 
$
\mathbb{P}[\exists p \leq n : |\widetilde{\mu}_k(p) - \mu_k| \geq B_{T_k(p)} ] \leq \frac{4n}{T}$.\end{lemm}

\begin{proof}
For any arm $k$ and positive integer $n$, Hoeffding inequality gives the following, classical inequality in MAB: $
\mathbb{P}[\exists p \leq n : |\hat{\mu}_k(p) - \mu_k| \geq \sqrt{\frac{2\log(T)}{T_k(p)}} ] \leq \frac{2n}{T}
$. It remains to bound the estimation error due to quantization.

Notice that $\sum_{j=1}^M (\widetilde{S}_k^j - \lfloor S_k^j \rfloor)$ is the sum of $M$ independent Bernoulli at each phase $p$. Hoeffding inequality thus also claims that $\mathbb{P}[|\sum_{j=1}^M(\widetilde{S}_k^j(p) - S_k^j(p))| \geq \sqrt{\frac{\log(T) M}{2}} ] \leq \frac{2}{T}$. As $T_k(p) \geq M$, it then holds $\mathbb{P}[\exists p \leq n : |\widetilde{\mu}_k^j(p) - \hat{\mu}_k^j(p)| \geq \sqrt{\frac{\log(T)}{2T_k(p)}} ] \leq \frac{2n}{T}$. Using the triangle inequality with this bound and the first Hoeffding inequality of the proof yields the final result.
\end{proof}

For both exploration and exploitation phases, we control the number of times an arm is pulled before being accepted or rejected.

\begin{prop}
\label{prop:explo}
With probability $1 - \mathcal{O} \left( \frac{K \log(T)}{T} +M \exp \left( - \frac{T_0}{K} \right) \right) $, every optimal arm $k$ is accepted after at most $\mathcal{O} \left( \frac{\log(T)}{(\mu_k - \mu_{(M+1)})^2} \right)$ pulls during exploration phases, and every sub-optimal arm $k$ is rejected after at most $\mathcal{O} \left( \frac{\log(T)}{(\mu_{(M)} - \mu_{k})^2} \right)$ pulls during exploration phases.
\end{prop}

\begin{proof}
\label{proof:explo}
With probability at least $1 - M \exp \left( - \frac{T_0}{K} \right)$, the initialization is successful, \ie all players have been assigned different ranks. The remaining of the proof is conditioned on that event.

As there are at most $\log_2(T)$ exploration-communication phases, $|\widetilde{\mu}_k(p) - \mu_k| \leq B_{T_k(p)}$ holds for any arm and phase with probability $1 - \mathcal{O} \left( \frac{K \log(T)}{T} \right)$ thanks to Lemma~\ref{lemm:concentration}. The remaining of the proof is conditioned on that event. 

We first consider an optimal arm $k$. Let $\Delta_k = \mu_k - \mu_{(M+1)}$ be the gap between the arm $k$ and the first sub-optimal arm. We assume $\Delta_k >0$ here, the case of equality holds considering $\frac{\log(T)}{0}=\infty$. Let $s_k$ be the first integer such that $4 B_{s_k} \leq \Delta_k$.

With $T_k(p) = \sum_{l=1}^p M_l 2^l$ the number of times an active arm has been pulled after the $p$-th exploration phase, it holds that \begin{equation}
\label{eq:Sp}
T(p+1) \leq 3 T(p) \qquad \text{ as $M_p$ is non-increasing.}
\end{equation}  For some $p \in \mathbb{N}$, $T(p-1) < s_k \leq T(p)$ or the arm $k$ is active at time $T$. In the second case, it is obvious that $k$ is pulled less than $\mathcal{O}(s_k)$ times.
Otherwise, the triangle inequality for such a $p$, for any active sub-optimal arm $i$, yields
$
\widetilde{\mu}_k(p) - B_{T_k(p)} \geq \widetilde{\mu}_i(p) + B_{T_i(p)}.
$

So the arm $k$ is accepted after at most $p$ phases. Using the same argument as in \citep{batchedbandits}, it holds $s_k = \mathcal{O}\left( \frac{\log(T)}{(\mu_k - \mu_{(M+1)})^2}  \right)$, and also for $T_k(p)$ thanks to Equation~\eqref{eq:Sp}. Also, $k$ can not be wrongly rejected conditionally on the same event, as it can not be dominated by any sub-optimal arm in term of confidence intervals.

\medskip

The proof for the sub-optimal case is similar if we denote $\Delta_k = \mu_{(M)} - \mu_{k}$.
\end{proof}

In the following, we keep the notation $t_k =\min \Big\lbrace \frac{c \log(T)}{\left( \mu_{k} - \mu_{(M)} \right)^2},\ T \Big\rbrace$, where $c$ is a universal constant such that with the probability considered in Proposition~\ref{prop:explo}, the number of exploration pulls before accepting/rejecting $k$ is at most $t_k$.

\medskip

For both exploration and exploitation phases, the decomposition used in the centralized case \citep{anantharam} holds because there is no collision during these two types of phases (conditionally on the success of the initialization phase):
\begin{equation}
\label{eq:regdecomposition}
R^{\text{explo}} = \sum\limits_{k > M} (\mu_{(M)} - \mu_{(k)}) T_{(k)}^{\text{explo}} + \sum\limits_{k \leq M} (\mu_{(k)} - \mu_{(M)}) (T^{\text{explo}}-T_{(k)}^{\text{explo}}),
\end{equation}

where $T^{\text{explo}} = \#\text{Explo}$ and $T_{(k)}^{\text{explo}}$ is the centralized number of time steps where the $k$-th best arm is pulled during exploration or exploitation phases.

\begin{lemm}
\label{lemma:exploregretopt1}
With probability $1 - \mathcal{O} \left( \frac{K \log(T)}{T} + M \exp \left( - \frac{T_0}{K} \right) \right)$, the following hold simultaneously:
\begin{enumerate}[label=\roman*)]
\item for a sub-optimal arm $k$, $(\mu_{(M)} - \mu_k) T_k^{\mathrm{explo}} = \mathcal{O} \left( \min \bigg\lbrace \frac{\log(T)}{\mu_{(M)} - \mu_k},\! \sqrt{T \log(T)} \bigg\rbrace \right).$
\item $\sum\limits_{k \leq M} (\mu_{(k)} - \mu_{(M)}) (T^{\mathrm{explo}} - T_{(k)}^{\mathrm{explo}}) = \mathcal{O} \left( \sum\limits_{k > M} \min \bigg\lbrace \frac{\log(T)}{\mu_{(M)} - \mu_{(k)}}, \sqrt{T \log(T)} \bigg\rbrace \right).$
\end{enumerate}
\end{lemm}

\begin{proof}
\label{proof:exploregretopt}
i) From Proposition~\ref{prop:explo}, $T_k^{\mathrm{explo}} \leq \mathcal{O} \left( \min \bigg\lbrace \frac{\log(T)}{(\mu_{(M)} - \mu_k)^2} , T \bigg\rbrace \right)$ with the considered probability, so $(\mu_{(M)} - \mu_k) T_k^{\mathrm{explo}} = \mathcal{O} \left( \min \bigg\lbrace \frac{\log(T)}{(\mu_{(M)} - \mu_k)},\ (\mu_{(M)} - \mu_k)T \bigg\rbrace \right)$. The function $\Delta \mapsto \min \bigg\lbrace \frac{\log(T)}{\Delta},\ \Delta T \bigg\rbrace$ is maximized for $\Delta = \sqrt{\frac{\log(T)}{T}}$ and its maximum is $\sqrt{T \log(T)}$. Thus, the inequality $\min \bigg\lbrace \frac{\log(T)}{\Delta},\ \Delta T \bigg\rbrace \leq \min \bigg\lbrace \frac{\log(T)}{\Delta},\ \sqrt{T \log(T)} \bigg\rbrace$ always holds for $\Delta \geq 0$ and yields the first point.

\medskip

ii) We (re)define the following: $\hat t_k$ the number of exploratory pulls before accepting/rejecting the arm $k$, $M_l$ the number of active player during the $l$-th exploration phase, $T(p) = \sum\limits_{l=1}^p 2^{l} M_l$ and $N$ the total number of exploration phases.

$T(p)$ describes the total number of exploration pulls processed at the end of the $p$-th exploration phase on every active arm for $p<N$. Since the $N$-th phase may remain uncompleted, $T(N)$ is then greater that the number of exploration pulls at the end of the $N$-th phase.

\medskip

With probability $1 - \mathcal{O} \left( \frac{K \log(T)}{T} + M \exp \left( - \frac{T_0}{K} \right) \right)$, the initialization is successful, any arm is correctly accepted or rejected and $\hat{t}_k \leq t_k$ for all $k$. The remaining of the proof is conditioned on that event.
We now decompose the proof in two main parts given by Lemmas~\ref{lemma:auxregret1} and \ref{lemma:auxregret2} proven below. 

\begin{lemm}
Conditionally on the success of the initialization phase and on correct estimations of all arms:
\label{lemma:auxregret1}
\begin{equation*}
\sum\limits_{k \leq M} (\mu_{(k)} - \mu_{(M)})(T^\text{explo} - T_{(k)}^{\text{explo}}) \leq \sum\limits_{j > M}\sum\limits_{k \leq M} \sum\limits_{p=1}^{N} 2^p (\mu_{(k)} - \mu_{(M)}) \mathds{1}_{\min( \hat{t}_{(j)} , \hat{t}_{(k)}) > T(p-1)}.
\end{equation*}
\end{lemm}

\begin{lemm}
Conditionally on the success of the initialization phase and on correct estimations of all arms:
\label{lemma:auxregret2}
\begin{equation*}
\sum\limits_{k \leq M} \sum\limits_{p=1}^{N} 2^p (\mu_{(k)} - \mu_{(M)}) \mathds{1}_{\min( \hat{t}_{(j)}, \hat{t}_{(k)}) > T(p-1)} \leq \mathcal{O}\left( \min \bigg\lbrace \frac{\log(T)}{\mu_{(M)} - \mu_{(j)}}, \sqrt{T \log(T)} \bigg\rbrace \right).
\end{equation*}
\end{lemm}

These two lemmas directly yield the second point in Lemma~\ref{lemma:exploregretopt1}.
\end{proof}

\begin{proof}[Proof of Lemma~\ref{lemma:auxregret1}.]
Let us consider an optimal arm $k$. During the $p$-th exploration phase, there are two possibilities:
\begin{itemize}
\item either $k$ has already been accepted, i.e., $\hat{t}_k \leq T(p-1)$. Then the arm $k$ is pulled the whole phase, \ie $K_p2^p$ times.
\item Or $k$ is still active. Then it is pulled $2^p$ times by each active player, i.e., it is pulled $M_p 2^p$ times in total. This means that it is not pulled $(K_p-M_p)2^p$ times.
\end{itemize}

From these two points, it holds that $
T_k^{\text{explo}} \geq T^{\textrm{explo}} - \sum\limits_{p=1}^{N} 2^p (K_p - M_p)\mathds{1}_{\hat{t}_k > T(p-1)}$.

Notice that $K_p - M_p$ is the number of active sub-optimal arms. By definition, $K_p - M_p = \sum\limits_{j > M} \mathds{1}_{\hat{t}_{(j)} > T(p-1)}$. We thus get that $T_k^{\text{explo}} \geq T^{\text{explo}} - \sum\limits_{j > M} \sum\limits_{p=1}^{N} 2^p \mathds{1}_{\min( \hat{t}_{(j)} , \hat{t}_k) > T(p-1)}$.

The double sum actually is the number of times a sub-optimal arm is pulled instead of $k$. 
This yields the result when summing over all optimal arms $k$.
\end{proof}

\begin{proof}[Proof of Lemma~\ref{lemma:auxregret2}.]
Let us define $A_j = \sum\limits_{k \leq M} \sum\limits_{p=1}^{N} 2^p (\mu_{(k)} - \mu_{(M)}) \mathds{1}_{\min( \hat{t}_{j} , \hat{t}_{(k)}) > T(p-1)}$ the cost associated to the sub-optimal arm $j$. Lemma~\ref{lemma:auxregret2} upper bounds $A_j$ for any sub-optimal arm $j$.

\medskip

Recall that $t_{(k)} = \min \left( \frac{c \log(T)}{\left( \mu_{(k)} - \mu_{(M)} \right)^2},\ T \right)$ for a universal constant $c$. The proof is conditioned on the event $\hat{t}_{(k)} \leq t_{(k)}$, so that if we define $\Delta(p) = \sqrt{\frac{c \log(T)}{T(p-1)}}$, the inequality $\hat{t}_{(k)} > T(p-1)$ implies $\mu_{(k)} - \mu_{(M)} < \Delta(p)$. We also write $N^j$ the first integer such that $\hat{t}_j \leq T(N^j)$. It follows:
\begin{align*}
A_j & \leq \sum\limits_{k \leq M} \sum\limits_{p=1}^{N^j} 2^p \Delta(p) \mathds{1}_{\hat{t}(k)>T(p-1)} \\
& \leq \sum\limits_{p=1}^{N^j} \Delta(p) \left( T(p) - T(p-1) \right) \hspace{4.8cm} \text{as } \sum\limits_{k \leq M} \mathds{1}_{\hat{t}(k)>T(p-1)} = M_p. \\
&  = c \log(T) \sum\limits_{p=1}^{N^j} \Delta(p) \left( \frac{1}{\Delta(p+1)} + \frac{1}{\Delta(p)} \right) \left( \frac{1}{\Delta(p+1)} - \frac{1}{\Delta(p)} \right) \\
& \leq (1+\sqrt{3}) c \log(T) \sum\limits_{p=1}^{N^j} (\frac{1}{\Delta(p+1)} - \frac{1}{\Delta(p)}) \hspace*{3.85cm} \text{thanks to Equation~\eqref{eq:Sp}.} \\
& \leq (1+\sqrt{3}) c \log(T) \frac{1}{\Delta(N^j +1)} \hspace*{5.2cm} \text{by convention, } \frac{1}{\Delta(1)} = 0.
\end{align*}

By definition of $N^j$, we have $t_j \geq T(N^j-1)$. Thus, $\Delta(N^j) \geq \sqrt{\frac{c \log(T)}{t_j}}$ and Equation~\eqref{eq:Sp} gives $\Delta(N^j+1) \geq \sqrt{\frac{c \log(T)}{3t_j}}$. It then holds $A_j \leq (3+\sqrt{3}) \sqrt{c\ t_j \log(T)}$. The result follows since $t_j = \mathcal{O}\left(\min\big\lbrace \frac{\log(T)}{(\mu_{(M)} - \mu_j)^2}, T \big\rbrace \right)$.
\end{proof}

Using the two points of Lemma~\ref{lemma:exploregretopt1}, along with Equation~\eqref{eq:regdecomposition}, yields Lemma~\ref{lemma:exploregretopt}.

\subsubsection{Communication cost}
We now focus on the $R^{\mathrm{comm}}$ term in Equation~\eqref{eq:regdec}. Lemma~\ref{lemma:commregret} states it is negligible compared to $\log(T)$ and has a significant impact on the regret only for small values of $T$.
\begin{lemm}
\label{lemma:commregret}
With probability $1 - \mathcal{O} \left( \frac{K \log(T)}{T} + M \exp \left( - \frac{T_0}{K} \right) \right)$, the following holds:$$R^{\mathrm{comm}} = \mathcal{O} \left( KM^3 \log^2 \left( \min \bigg\lbrace \frac{\log(T)}{(\mu_{(M)}-\mu_{(M+1)})^2},\ T \bigg\rbrace \right) \right).$$
\end{lemm}

\begin{proof}
As explained in Section~\ref{Phase comm}, the length of the communication phase $p \in [N]$ is at most $K M^2 (p+1)$, where $N$ is the number of exploration phases. The cost of communication is then smaller than
$KM^3 \sum_{p=1}^{N} (p+1) \leq \mathcal{O} \left( KM^3 N^2 \right)$.
Proposition~\ref{prop:explo} in Appendix~\ref{app:exploregret}, claims with the considered probability that $N$ is at most
$\mathcal{O} \Big( \log \Big( \min \bigg\lbrace \frac{\log(T)}{(\mu_{(M)} - \mu_{(M+1)})^2},\ T \bigg\rbrace \Big)  \Big)$, which yields Lemma~\ref{lemma:commregret}.
\end{proof}

\subsubsection{Total regret}

The choice $T_0 = \lceil K \log(T) \rceil$ along with Lemmas~\ref{lemma:musicalchair}, \ref{lemma:exploregretopt} and \ref{lemma:commregret} claim that a bad event occurs with probability at most $\mathcal{O} \left( \frac{K \log(T)}{T} + \frac{M}{T}\right)$. The average regret due to bad events is thus upper bounded by $\mathcal{O}(KM \log(T))$. Using these lemmas along with Equation~\eqref{eq:regdec} finally yields the bound in Theorem~\ref{thm:upperbound}.

\subsection{Experiments}
\label{sec:expe}
We compare in Figure~\ref{fig:synchcommexp} the empirical performances of \algoone with the \textsc{mct}op\textsc{m} algorithm\citep{besson} on generated data\footnote{The code is available at \url{https://github.com/eboursier/sic-mmab}.}. 
We also compared with the \textsc{m}usical\textsc{c}hairs algorithm \citep{musicalchair}, but its performance was irrelevant and out of scale. This is mainly due to its scaling with $1/\Delta^2$, besides presenting large constant terms in its regret. Also, its main advantage comes from its scaling with $M$, which is here small for computational reasons. All the considered regret values are averaged over 200 runs.
The experiments are run with Bernoulli distributions. Thus, there is no need to quantize the sent statistics and a tighter confidence bound $B_s = \sqrt{\frac{2\log(T)}{s}}$ is used.

Figure~\ref{fig:synchcommexp1} represents the evolution of the regret for both algorithms with the following problem parameters: $K=9$, $M=6$, $T=5 \times 10^5$. The means of the arms are linearly distributed between $0.9$ and $0.89$, so the gap between two consecutive arms is $1.25 \times 10^{-3}$. The switches between exploration and communication phases for \algoone are easily observable. A larger horizon (near $40$ times larger) is required for \algoone to converge to a constant regret, but this alternation between the phases could not be visible for such a value of $T$.

Figure~\ref{fig:synchcommexp2} represents the evolution of the final regret as a function of the gap $\Delta$ between two consecutive arms in a logarithmic scale. The problem parameters $K$, $M$ and $T$ are the same. Although \textsc{mct}op\textsc{m} seems to provide better results with larger values of $\Delta$, \algoone seems to have a smaller dependency in $1/\Delta$. This confirms the theoretical results claiming that \textsc{mct}op\textsc{m} scales with $\Delta^{-2}$ while \algoone scales with $\Delta^{-1}$. This can be observed on the left part of Figure \ref{fig:synchcommexp2} where the slope for \textsc{mct}op\textsc{m} is approximately twice as large as for \algoone[.]Also, a different behavior of the regret appears for very low values of $\Delta$ which is certainly due to the fact that the regret only depends on $T$ for extremely small values of $\Delta$ (minmax regret).

   
\begin{figure}[H]
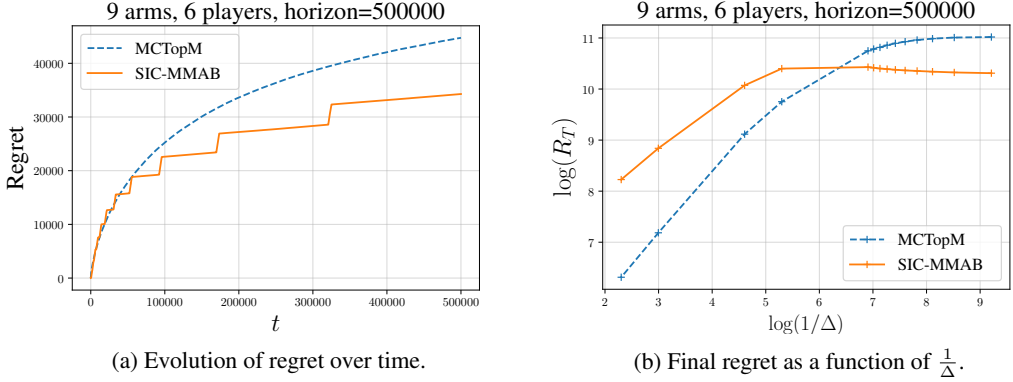

   \begin{subfigure}{0.5\textwidth}
   \centering
     \resizebox{\linewidth}{!}{\input{synchcommexp1.pgf}}
         \caption{Evolution of regret over time.}    \label{fig:synchcommexp1}
   \end{subfigure}
   \begin{subfigure}{0.5\textwidth}
     \centering
         \resizebox{\linewidth}{!}{\input{synchcommexp2.pgf}}
         \caption{Final regret as a function of $\frac{1}{\Delta}$.} \label{fig:synchcommexp2}
         \end{subfigure}
         \caption{Performance comparison between \algoone and \textsc{mct}op\textsc{m} algorithms. \label{fig:synchcommexp}}
\end{figure}
%
\section{Complementary material for Section~\ref{sec:nosens2}}
\label{app:alg30}
\subsection{\algothree description}
\label{app:alg3}

This section thoroughly describes \algothree algorithm. Its pseudocode is given in Algorithm~\ref{algo:nosens2} below.
\begin{algorithm}[h]
\caption{\label{algo:nosens2}\algothree algorithm}
\hspace*{\algorithmicindent} \textbf{Input:} $T^j$ (personal horizon)
\begin{algorithmic}[1]
\STATE $p \gets 1$, $\texttt{Fixed} \gets -1$ and initialize $\texttt{Preferences},\ \texttt{Occupied}$ as empty lists
\STATE $\mathbf{T}, \mathbf{T}^{\text{temp}}, \mathbf{S}, \mathbf{S}^{\text{temp}} \gets \text{Zeros}(K)$ and define $\mathbf{L}$ as a vector of $K$ elements equal to $\infty$
\STATE $r_{\inf}[k] \gets 0$ and $r_{\sup}[k] \gets 1$ for every arm $k$ \COMMENT{Initialize the confidence intervals} 
\STATE \textbf{Exploration Phase:} \COMMENT{$B^j(t) = 2\sqrt\frac{6 \ K \log(T^j)}{t}$ here}
\WHILE{\texttt{Fixed} $=-1$}
	\STATE Pull $k \sim \mathcal{U}([K])$; $T^{\text{temp}}[k] \gets T^{\text{temp}}[k] + 1$ and $T[k] \gets T[k] + 1$
	\STATE $S^{\text{temp}}[k] \gets S^{\text{temp}}[k] + r_k(t)$ and $S[k] \gets S[k] + r_k (t)$
	\STATE For all arms $i$, $r_{\inf}[i] \gets \left(\frac{S[i]}{T[i]} - B^j(t)\right)^+$ and $r_{\sup}[i] \gets \min \left( \frac{S[i]}{T[i]} + B^j(t), 1 \right)$
	\STATE $L[k] \gets \min \left( \frac{2e \log(T^j)}{r_{\inf}[k]}, \ L[k] \right)$
	\STATE \algorithmicif\ $k = \texttt{Preferences}[p]$ and $r_k(t)>0$\ \algorithmicthen\ \texttt{Fixed} $\gets k$ \algorithmicendif	\COMMENT{no collision\\ \hfill on the arm to exploit} \label{alg3:point0}
	  \STATE \algorithmicif\ $\texttt{Preferences}[p] \in \texttt{Occupied}$\ \algorithmicthen\ $p \gets p + 1$\ \algorithmicendif \COMMENT{exploited by another player} \label{alg3:point1}
	 
	 \IF[end of sliding window]{$T^{\text{temp}}[k] \geq L[k]$} 	\label{alg3:rule10}
	 \STATE \algorithmicif\ $S^{\text{temp}}[k] = 0$\ \algorithmicthen\ Add $k$ to $\texttt{Occupied}$\ \algorithmicendif		 \COMMENT{estimate that $k$ is occupied}
	 \STATE Reset $S^{\text{temp}}[k], T^{\text{temp}}[k] \gets 0$
	 \ENDIF \label{alg3:rule11}
	 
	 \IF{for some active arm $i$ and all other active arms $l$, $r_{\inf}[i] > r_{\sup}[l]$} \label{alg3:rule20}
	 \STATE Add $i$ to $\texttt{Preferences}$ (last position) \COMMENT{$i$ is better than any other active arm}
	 \ENDIF \label{alg3:rule21}
	 \STATE \algorithmicif\ {there is some $l$ not in $\texttt{Preferences}[1:p], \text{ such that }  r_{\inf}[l] > r_{\sup}[\texttt{Preferences}[p]]$}\ \label{alg3:rule30} \\
	 \algorithmicthen\ add $\texttt{Preferences}[p]$ to \texttt{Occupied}
	 \STATE \algorithmicendif \label{alg3:rule31}  \COMMENT{the mean of the available best arm has significantly dropped}
\ENDWHILE
\vspace{0.2cm}

\STATE \textbf{Exploitation Phase:} Pull \texttt{Fixed} until $T^j$

\end{algorithmic}
\end{algorithm}
We first describe the rules explaining when a player adds an arm to \texttt{Occupied} or \texttt{Preferences.}

\medskip

An arm $k$ is added to \texttt{Occupied} (it may already be in \texttt{Preferences}) if only $0$ rewards have been observed during a whole block of $L_k^j$ pulls on arm $k$ for player $j$. Such a block ends when $L_k^j$ observations have been gathered on arm $k$ and a new block is then restarted. $L_k^j$ is an estimation of the required number of successive $0$ to observe before considering an arm as occupied with high probability. Its value at stage $t+\tau_j$, $L_k^j(t)$, is thus constantly updated using the current estimation of a lower bound of $\mu_k$:
\begin{equation}
\label{eq:updateTf}
L_{k}^j(t+1) \gets \min \left( \frac{2e \log(T^j)}{\left(\hat{r}_k^j(t+1) - B^j(t+1)\right)^+}, \ L_{k}^j(t) \right)
\quad \text{and } L_{k}^j(0) = + \infty,
\end{equation}
where $\hat{r}_k^j(t)$ is the empirical mean reward  on the arm $k$ at stage $t+\tau_j$, $B^j(t) = 2\sqrt\frac{6 \ K \log(T^j)}{t}$, $x^+ = \max(x, 0)$ and $\frac{2e \log(T^j)}{0} = + \infty$. This rule is described at lines~\ref{alg3:rule10}-\ref{alg3:rule11} in Algorithm~\ref{algo:nosens2}.

\medskip

An active arm $k$ is added to \texttt{Preferences} (at last position) if it is better than any other active arm, in term of confidence interval.
This rule is described at lines~\ref{alg3:rule20}-\ref{alg3:rule21} in Algorithm~\ref{algo:nosens2}.

\medskip

Another rule needs to be added to handle the possible case of an arm in \texttt{Preferences} already exploited by another player. As soon as an arm $k$ in \texttt{Preferences} becomes worse (in terms of confidence intervals) than an active arm or an arm with a higher index in \texttt{Preferences}, then $k$ is added to \texttt{Occupied}. This rule is described at lines~\ref{alg3:rule30}-\ref{alg3:rule31} in Algorithm~\ref{algo:nosens2}.

\medskip

Following these rules, as soon as there is an arm in \texttt{Preferences}, player $j$ tries to occupy the $p$-th arm in $\texttt{Preferences}$ (starting with $p=1$), yet she still continues to explore. As soon as she encounters a positive reward on it, she occupies it and starts the exploitation phase. If she does not end up occupying an optimal arm, this arm will be added to \texttt{Occupied} at some point. The player then increments $p$ and tries to occupy the next available best arm. This point is described at lines~\ref{alg3:point0}-\ref{alg3:point1} in Algorithm~\ref{algo:nosens2}. 
Notice that \texttt{Preferences} can have more than $p$ elements, but the player must not exploit the $q$-th element of \texttt{Preferences} with $q > p$ yet as it can lead the player in exploiting a sub-optimal arm.

\subsection{Theoretical analysis}
\label{app:alg3proofs}
\subsubsection{Auxiliary lemmas}
This section is devoted to the proof of Theorem~\ref{thm:dynamicregret}. It first proves the first point of Lemma~\ref{lemma:nosensoccupied}.

\begin{proof}[Proof of Lemma~\ref{lemma:nosensoccupied}.1.]
\label{proof:nosensoccupied} We first introduce $Z_t \coloneqq X_k(t+\tau_j)(1-\eta_k(t+\tau_j)) \mathds{1}_{\pi^j(t+\tau_j)=k}$ and $p_t \coloneqq \mathbb{E} [Z_t]$. Notice that $p_t \leq \frac{1}{K}$ because $\mathds{1}_{\pi^j(t+\tau_j)=k}$ is a Bernoulli of parameter $\frac{1}{K}$ in the exploration phase. Chernoff bound states that:
\begin{equation*}
\mathbb{P}\Big[ \sum\limits_{t'=1}^t (Z_{t'} - \mathbb{E}[Z_{t'}]) \geq t \delta \Big] \leq \min\limits_{\lambda>0} e^{-\lambda t \delta} \ \mathbb{E}\big[ \prod\limits_{t'=1}^{t} e^{\lambda(Z_{t'} - \mathbb{E}[Z_{t'}])}\big].
\end{equation*}

By convexity, $e^{\lambda z} \leq 1+ z(e^{\lambda}-1)$ for $z \in [0,1]$. It thus holds:
\begin{align*}
\mathbb{E}\Big[e^{\lambda (Z_t-\mathbb{E}[Z_t])}\Big] & \leq e^{-\lambda p_t} \left( 1 + p_t(e^\lambda -1)\right) \leq e^{-\lambda p_t} e^{p_t(e^\lambda -1)} \qquad \text{ as } 1+x \leq e^x .\\
& \leq e^{p_t (e^\lambda -1 -\lambda)} \leq e^{\frac{e^\lambda -1-\lambda}{K}} \hspace{0.9cm} \text{ as } p_t\leq\frac{1}{K} \text{ and } e^\lambda -1-\lambda\geq 0.
\end{align*}

It can then be deduced:
\begin{align*}
\mathbb{P}\Big[ \sum\limits_{t'=1}^t (Z_{t'} - \mathbb{E}[Z_{t'}])\geq t \delta  \Big] & \leq \min\limits_{\lambda>0} e^{-\lambda t \delta} e^{t \frac{e^{\lambda}-1-\lambda}{K}}. & \qquad \text{For } \lambda = \log(1 + K\delta): \\
& \leq \exp\left( -\frac{t}{K} h(K\delta) \right) & \qquad \text{with } h(u) = (1+u) \log(1+u) - u .
\end{align*}

Similarly, we show for the negative error: $\mathbb{P}\Big[ \sum\limits_{t'=1}^t (Z_{t'} - \mathbb{E}[Z_{t'}])\leq - t \delta  \Big] \leq \exp\left( -\frac{t}{K} h(-K\delta) \right)$.
\medskip

Either $t \leq \frac{16}{3}K\log(T^j)$ and the desired inequality holds almost surely, or $K\delta < 1$ with $\delta = \sqrt{\frac{16 \log(T^j)}{3tK}}$. As $h(x) \geq \frac{3x^2}{8}$ for $|x|<1$, it then holds
\begin{align*}
\mathbb{P}\Big[ \Big|\sum\limits_{t'=1}^t (Z_{t'} - \mathbb{E}[Z_{t'}]) \Big| \geq t\delta \Big] \leq 2e^{-\frac{3t (K\delta)^2}{8K}} \qquad \text{and after multiplication with } \frac{K}{t}:
\end{align*}
\begin{equation}
\label{eq:nosens2explo1part1}
  \mathbb{P}\Bigg[  \Big| \frac{K}{t} {\mathlarger\sum_{t'= 1 + \tau_j}^{t+\tau_j}} X_k(t')(1-\eta_k(t')) \mathds{1}_{\pi^j(t')=k} - \gamma_j(t) \mu_k \Big| \geq \sqrt{\frac{16K \log(T^j)}{3t}}\Bigg] \leq \frac{2}{(T^j)^2}.
\end{equation}

Chernoff bound also provides a confidence interval on the number of pulls on a single arm:
\begin{equation}
\label{eq:chernoffpulls}
\mathbb{P}\Bigg[ \Big|T_k^j(t) - \frac{t}{K} \Big| \geq \sqrt{\frac{6t \log(T^j)}{K}}\Bigg] \leq \frac{2}{(T^j)^2}.
\end{equation}

From Equation~\eqref{eq:chernoffpulls}, it can be directly deduced that $\mathbb{P}\Big[|\frac{K T_k^j(t)}{t} - 1| \geq \sqrt{\frac{6K \log(T^j)}{t}}\Big] \leq \frac{2}{(T^j)^2}$. As $\hat{r}^j_k(t) \leq 1$,
\begin{equation}
\label{eq:nosens2explo1part2}
\mathbb{P}\Bigg[\Big| \frac{K T_k^j(t)}{t}\hat{r}_k^j(t) - \hat{r}_k^j(t)\Big| \geq \sqrt{\frac{6K \log(T^j)}{t}}\Bigg] \leq \frac{2}{(T^j)^2}. 
\end{equation}

As $\frac{K T_k^j(t)}{t}\hat{r}_k^j(t) = \frac{K}{t} \sum\limits_{t'= 1 + \tau_j}^{t+\tau_j} X_k(t')(1-\eta_k(t')) \mathds{1}_{\pi^j(t')=k}$, using the triangle inequality with Equations~\eqref{eq:nosens2explo1part1} and \eqref{eq:nosens2explo1part2} finally yields $
	\mathbb{P}\Big[|\hat{r}^j_k(t) - \gamma^j(t) \mu_k| \geq 2\sqrt\frac{6 \ K \log(T^j)}{t}\Big] \leq \frac{4}{(T^j)^2}$.
\end{proof}

The second point of Lemma~\ref{lemma:nosensoccupied} is proved below.

\begin{proof}[Proof of Lemma~\ref{lemma:nosensoccupied}.2.]
The previous point gives that with probability $1- \mathcal{O}\left( \frac{K}{T^j} \right)$, player $j$ correctly estimated all the free arms until stage $T$. The remaining of the proof is conditioned on this event.
We also assume that $t^0$ is the first stage where $k$ is occupied for the proof. The general result claimed in Lemma~\ref{lemma:nosensoccupied} directly follows.

When $t^0$ is small, the second case will happen, \ie the number of pulls on the arm $k$ is small and its average reward can quickly drop to $0$. When $t^0$ is large, $\gamma_j(t) \mu_k$ is tightly estimated so that $L_k^j$ is small. Then, the first case will happen, \ie the arm $k$ will be quickly detected as occupied.

\medskip

a) We first assume $t^0 \leq 12 K \log(T^j)$. The empirical reward after $T_k^j(t) \geq T_k^j(t^0)$ pulls is $\hat{r}_k^j(t) = \frac{\hat{r}_k^j(t^0) T_k^j(t^0)}{T_k^j(t)}$, because all pulls after the stage $t^0+\tau_j$ will return $0$ rewards.
However, using Chernoff bound as in Equation~\eqref{eq:chernoffpulls}, it appears that if $t^0 \leq 12 K \log(T^j)$ then $T_k^j(t^0) \leq 18 \log(T^j)$ with probability $1-\mathcal{O}\left( \frac{1}{T^j}\right)$, so $\hat{r}_k^j(t) \leq \frac{18 \log(T^j)}{T_k^j(t)}$.

Conditionally on the correct estimations of the arms, there is at least an unoccupied arm $i$ with $\mu_i \leq \mu_k$. Therefore with $t_i = \frac{72K e \log(T^j)}{\mu_i^2}$, as $t_i \geq 12 K \log(T^j)$, Chernoff bound guarantees that the following holds, with probability at least $1 - \frac{2}{T^j}$,
\begin{equation}
\label{eq:chernoff2}
\frac{3t_i}{2K} \geq T_k^j(t_i) \geq \frac{t_i}{2K} = \frac{36 e \log(T^j)}{\mu_i^2}.
\end{equation}
This gives that $\hat{r}_k^j(t_i) \leq \frac{\mu_i}{2e}$. After stage $\tau_j+ \frac{d' K \log(T^j)}{\mu_i^2}$, where $d'$ is some universal constant, the error bounds of both arms are upper bounded by $\frac{\mu_i}{8e}$. The confidence intervals would then be disjoint for the arms $k$ and $i$. So $k$ will be detected as worse than $i$ after a time at most $\mathcal{O} \left( \frac{K \log(T)}{\mu_i^2} \right)$ as $T^j \leq T$.

\medskip

b) We now assume that $12 K \log(T^j) \leq t^0 \leq \frac{24 \lambda K\log(T^j)}{\mu_k^2}$ with $\lambda = 16e^2$. It still holds $\hat{r}_k^j(t) = \frac{\hat{r}_k^j(t^0) T_k^j(t^0)}{T_k^j(t)}$. 
Correct estimations of the free arms are assumed in this proof, so in particular
\begin{equation}
\label{eq:empiricalrewardbond}
\hat{r}_k^j(t) \leq \frac{(\mu_k + B^j(t^0)) T_k^j(t^0)}{T_k^j(t)}.
\end{equation}

As in Equation~\eqref{eq:chernoff2}, it holds that $T_k^j(t^0) \leq \frac{3 t^0}{2K}$ with probability $1 - \mathcal{O} \left( \frac{1}{T^j} \right)$ and thus $B^j(t^0) \leq 6\sqrt{\frac{\log(T^j)}{T_k^j(t^0)}}$. Also, $T_k^j(t) \geq \frac{d \log(T^j)}{2 \mu_i \mu_k}$ for $t = d\frac{K \log(T^j)}{\mu_i^2}$. Equation~\eqref{eq:empiricalrewardbond} then becomes
\begin{equation*}
\hat{r}_k^j(t)  \leq  \frac{\mu_k T_k^j(t^0)}{T_k^j(t)}  +  \frac{B^j(t^0) T_k^j(t^0)}{T_k^j(t)} \leq \frac{36 \lambda}{d}\mu_i  +  \frac{6\sqrt{ T^j_k(t^0) \log(T^j)}}{T^j_k(t)}
 \leq \left( \frac{36 \lambda}{d}  +  \frac{72 \sqrt{\lambda}}{d} \right) \mu_i.
\end{equation*}

Thus, for a well chosen $d$, the empirical reward verifies $\hat{r}_k^j(t) \leq \frac{\mu_i}{2e}$. We then conclude as for the first case that the arm $k$ would be detected as worse than the free arm $i$ after a time $\mathcal{O} \left( \frac{K \log(T)}{\mu_i^2} \right)$.

\medskip

c) The last case corresponds to $t^0 > \frac{24 \lambda K\log(T^j)}{\mu_k^2}$. It then holds $B^j(t^0) \leq \frac{\mu_k}{\sqrt{\lambda}} = \frac{\mu_k}{4e}$.

By definition, $L_k^j \leq \frac{2e \log(T^j)}{\hat{r}_k^j - B^j(t)}$.
Conditionally on the correct estimation of the free arms, it holds that $\gamma_j(t)\mu_k - 2 B^j(t) \leq \hat{r}_k^j - B^j(t) \leq \mu_k$. So with the choice of $L_k^j$ described by Equation~\eqref{eq:updateTf}, as long as $k$ is free,
\begin{equation}
\label{eq:boundTf}
\begin{aligned}
\frac{2e \log(T^j)}{\mu_k} & \leq L_k^j & \leq \frac{2e \log(T^j)}{\gamma_j(t)\mu_k - 2 B^j(t)} \\
& & \leq \frac{2e^2 \log(T^j)}{\mu_k - 2e B^j(t)}.
\end{aligned}
\end{equation}

As $B^j(t^0) \leq \frac{\mu_k}{4e}$, it holds that $L_k^j(t^0) \leq \frac{4e^2 \log(T^j)}{\mu_k}$. Since $L_k^j$ is non-increasing by definition, this actually always holds for any $t$ larger than $ t^0$.

From that point, Equation~\eqref{eq:chernoff2} gives that with probability $1 - \mathcal{O} \left( \frac{1}{T^j} \right)$, the arm $k$ will be pulled at least $2L_k^j$ times between stage $t^0 + 1$ and $t^0 + 24K L_k^j$ with probability $1-\mathcal{O}\left( \frac{1}{T^j}\right)$. Thus, a whole block of $L_k^j$ pulls receiving only $0$ rewards on $k$ happens before stage $t^0 + 24K L_k^j$.

The arm $k$ is then detected as occupied after a time  $\mathcal{O}\left( \frac{K \log(T^j)}{\mu_k} \right)$ from $t^0$, leading to the result.
\end{proof}


\begin{lemm}
\label{lemma:nosensoccupied2}
At any stage, no free arm $k$ is falsely detected as occupied by player $j$ with probability $1 - \mathcal{O}\left( \frac{K}{T^j} \right)$.
\end{lemm}

\begin{proof}
\label{proof:nosensoccupied2}
As shown above, with probability $1- \mathcal{O}\left( \frac{K}{T^j} \right)$, player $j$ correctly estimated the average rewards of all the free arms until stage $T$. The remaining of the proof is conditioned on that event. As long as $k$ is free, it can not become dominated by some arm that was not added to \texttt{Preferences} before $k$, so it can not be added to \texttt{Occupied} from the rule given at lines~\ref{alg3:rule30}-\ref{alg3:rule31} in Algorithm~\ref{algo:nosens2}.

For the rule of lines~\ref{alg3:rule10}-\ref{alg3:rule11}, Equation~\eqref{eq:boundTf} gives that
\begin{equation}
\label{eq:nosensoccupied21}
L_k^j(t') \geq \frac{2e \log(T^j)}{\mu_k} \qquad \text{ at each stage } t' \leq t.
\end{equation}

As in Appendix~\ref{proof:musicalchair}, the probability of detecting $L$ successive $0$ rewards on a free arm $k$ is then smaller than  $\left( 1-\frac{\mu_k}{e} \right)^{L} \leq \exp \left( -\frac{L \mu_k}{e} \right)$.

Using this along with Equation~\eqref{eq:nosensoccupied21} yields that with probability $1 - \mathcal{O} \left( \frac{1}{(T^j)^2} \right)$, at least one positive reward will be observed on arm $k$ in a single block. The union bound over all blocks yields the result.
\end{proof}

Finally, Lemma~\ref{lemma:nosens2} yields that, after some time, any player starts exploiting an arm while all the better arms are already occupied by other players.

\begin{lemm}
\label{lemma:nosens2}
We denote $\bar{\Delta}_{(k)} = \min\limits_{i=1,..., k}(\mu_{(i)} - \mu_{(i+1)})$. With probability $1 - \mathcal{O}\left( \frac{K}{T^j} \right)$, it holds that for a single player $j$, there exists $k_j$ such that after a stage at most $\bar{t}_{k_j} + \tau_j$, she is exploiting the $k_j$-th best arm and all the better arms are also exploited by other players, where $\bar{t}_{k_j} = \mathcal{O}\left( \frac{K \log(T)}{\bar{\Delta}_{(k_j)}^2} + k_j \frac{K \log(T)}{\mu_{(k_j)}} \right)$.
\end{lemm}

\begin{proof}
\label{proof:lemmanosens2}
Player $j$ correctly estimates all the arms until stage $T$, with probability $1 - \mathcal{O}\left(\frac{K}{T^j} \right)$. The remaining of the proof is conditioned on that event. We define $\bar{t}_{i} = \frac{cK \log(T^j)}{\bar{\Delta}_{(i)}^2} + i \frac{cK \log(T^j)}{\mu_{(i)}}$ for some universal constant $c$ and $k_j$ (random variable) defined as
\begin{equation}
\label{eq:choicek}
k_j = \min \Big\{ i \in [K] \ | \ i \text{-th best arm not exploited by another player at stage }  \bar{t}_i + \tau_j \Big\}.
\end{equation}

$k_j^*$ ($k_j$-th best arm) is the best arm not exploited by another player (than player $j$) after the stage $\bar{t}_{k_j} + \tau_j$. The considered set is not empty as $M \leq K$.

Lemma~\ref{lemma:nosensoccupied2} gives that with probability $1 - \mathcal{O} \left( \frac{K}{T^j} \right)$, $k_j^*$ is not falsely detected as occupied until stage $T$. All arms below $k_j^*$ will be detected as worse than $k_j^*$ after a time $\frac{d K \log(T^j)}{\bar{\Delta}_{(k_j)}^2}$ for some universal constant $d$.

\medskip

By definition of $k_j$, any arm $i^*$ better than $k_j^*$ is already occupied at stage $\bar{t}_i + \tau_j$. Lemma~\ref{lemma:nosensoccupied}, gives that with probability $1 - \mathcal{O}\left( \frac{1}{T^j} \right)$, either $i^*$ is detected as occupied after stage $\bar{t}_i + \tau_j+ \frac{d'K \log(T^j)}{\mu_{(i)}}$ or dominated by $k_j^*$ after stage $\frac{d_2 K \log(T^j)}{\bar{\Delta}_{(k_j)}^2} + \tau_j$ for some universal constants $d'$ and $d_2$. 

Thus the player detects the arm $k_j^*$ as optimal and starts trying to occupy $k_j^*$ at a stage at most $\tilde{t} = \max\Big(\bar{t}_{k_j-1} + \frac{d'K \log(T^j)}{\mu_{(k_j)}}, \max(d, d_2) \frac{K \log(T^j)}{\bar{\Delta}_{(k_j)}^2} \Big) + \tau_j$ with probability $1 - \mathcal{O}\left( \frac{K}{T^j} \right)$ (where $\bar{t}_0 = 0$). 

Using similar arguments as for Lemma~\ref{lemma:nosensoccupied2}, player $j$ will observe a positive reward on $k_j^*$ with probability $1-\mathcal{O}\left( \frac{1}{T^j} \right)$ after a stage at most $\tilde{t} + \frac{d'_2K \log(T^j)}{\mu_{(k_j)}}$ for some constant $d'_2$, if $k_j$ is still free at this stage. With the choice $c=\max(d, d_2, d' + d'_2)$, this stage is smaller than $\bar{t}_{k_j}$ and $k_j^*$ is then still free. Thus, player $j$ will start exploiting $k_j^*$ after stage at most $\bar{t}_{k_j}$ with the considered probability.
\end{proof}

\subsubsection{Regret in dynamic setting}
\label{sec:dynamicsetting}

\begin{proof}[Proof of Theorem~\ref{thm:dynamicregret}.]
Lemma~\ref{lemma:nosens2} states that a player only needs an exploration time bounded as $\mathcal{O}\Big(\frac{K \log(T)}{\bar{\Delta}_{(k)}^2} + k \frac{K \log(T)}{\mu_{(k)}}\Big)$ before starting exploiting, with high probability. Furthermore, the better arms are already exploited when she does so. Thus, the exploited arms are the top-$M$ arms. The regret is then upper bounded by twice the sum of exploration times (and the low probability events of wrong estimations), as a collision between players can only happen with at most one player in her exploitation phase.

The regret incurred by low probability events mentioned in Lemma~\ref{lemma:nosens2} is in $\mathcal{O}(KM^2)$ and is thus dominated by the exploration regret.
\end{proof}
%
\section{No Sensing: communication through synchronization }
\label{sec:nosens1}

This section focuses on the static No Sensing model. First of all, we claim that a communication  protocol similar to the one of \algoone can be devised here, under a mild extra assumption: a lower bound $\mu_{\min}$ of the average rewards $\mu_k$ is known\footnote{Actually, a lower bound of $\mathbb{P}[X_k>0]$ is enough. We instead use $\mu_{\min}$, as $\mathbb{P}[X_k>0] \geq \mu_k$.}. Indeed, in the \sensingtwo model, a bit is  sent through a single collision. Without sensing,  it can be done with probability $1- \frac{1}{T}$ using $\frac{\log(T)}{\mu_{\text{min}}}$ time steps. This adds a  multiplicative factor of $\frac{\log(T)}{\mu_{\text{min}}}$ to the communication regret\footnote{The length of the Musical chairs and the estimation protocol in the initialization will also be respectively multiplied by $\frac{1}{\mu_{\min}}$ and $\frac{\log(T)}{\mu_{\min}}$.}, which would then dominate the new initialization regret. So, \algoone can be easily adapted for the No Sensing model into the \algooneadapted algorithm with a regret scaling as
\begin{equation}
\label{eq:algooneadapted}
\mathcal{O} \left( {\mathlarger\sum_{k > M}} \frac{\log(T)}{\mu_{(M)} - \mu_{(k)}} +  
\frac{KM^3 \log(T)}{\mu_{\min}} \log^2 \left( \frac{\log(T)}{(\mu_{(M)}-\mu_{(M+1)})^2} \right) \right).
\end{equation}

The exploration regret is still similar to the centralized algorithm, but the communication cost is no longer sub-logarithmic. In this section, we introduce an alternative algorithm for the No sensing setting. It also relies on a communication protocol, but with more limited information, which thus incurs a much better dependency in $M$ as well as a logarithmic regret. 

In the No Sensing setting, the \textsc{selfish} strategy, where all players follow independently \ucb seems to perform well (on generated data) but appears to incur a linear regret with some constant probability \citep{besson}. In Appendix \ref{sec:selfish}, the discussion about the \textsc{selfish} algorithm is extended and some reasons of its failure are explained, using algebraic arguments (Lindemann-Weierstrass Theorem).

\subsection{Adapted communication protocol}

The algorithm \algotwo is formally described in Appendix~\ref{app:alg2}. It relies on several subroutines that are detailed in the next section. 
Similarly to \algoone[,]it starts with an initialization phase to estimate $M.$
It then alternates between exploration and communication phases, but the goal of the communication phases is here to communicate to other players that an arm is optimal or sub-optimal (instead of transmitting statistics). This allows to share common sets of active arms and players. Protocols to declare such arms and to detect declarations from other players are detailed in Appendix~\ref{sec:declarebad}. The algorithm then ends with an exploitation phase.

An additional assumption is required for \algotwo and is quite similar to an assumption made by \citet{gabor} for the No Sensing model.
\begin{hyp}
\label{hyp:mumin}
A lower bound of $\mu_{(K)}$ is known to all players: $0 < \mu_{\min} \leq \min\limits_{i} \mu_i$.
\end{hyp}

The regret incurred by \algotwo is given by Theorem~\ref{thm:algo2}. Its proof is given in Appendix~\ref{app:alg2proofs}.

\begin{thm}
\label{thm:algo2}
With the choice $T_c = \lceil \frac{\log(T)}{\mu_{\min}}\rceil$ for the initialization, \algotwo has a regret scaling as
$$
\mathbb{E}[R_T] = \mathcal{O} \left(  {\mathlarger\sum_{k > M}} \min\bigg\lbrace \frac{M \log(T)}{\mu_{(M)} - \mu_{(k)}} , \sqrt{MT \log(T)}\bigg\rbrace + \frac{MK^2}{\mu_{\min}} \log(T) \right).
$$
\end{thm}

\subsection{Algorithm description}
\label{app:alg2}

\algotwo algorithm is described in this section. 
We use the same definitions for $M_p$ and $K_p$ as in Section~\ref{sec:synchcomm}.

\subsubsection{Initialization phase}
\label{sec:nosensestimate}

The objective of the initialization phase is to estimate $M$. First, each player follows the Musical Chairs algorithm for a time $K T_c$ with $T_c \coloneqq \lceil \frac{\log(T)}{\mu_{\min}}\rceil$. The algorithm in the No Sensing setting is given by Pseudocode~\ref{algo:MC}, Appendix~\ref{app:alg1}. 
The second protocol of the initialization is then the same as for \algoone[,] but instead of a single time step, a number of $T_c$ time steps is needed to correctly transmit a bit with probability $1 - \frac{1}{T}$. The detailed protocol is given by Pseudocode~\ref{alg2:estimm}. 

\begin{pseudocode}[H]
\vspace{-0,5cm}
\begin{algorithm}[H]
\caption*{\textbf{\estimmnosens Protocol}}
\hspace*{\algorithmicindent} \textbf{Input:} $k$ (external rank), $T_c$ (time to send a bit) \\
\hspace*{\algorithmicindent} \textbf{Output:} $M$ (estimated number of players)
\begin{algorithmic}[1]
\STATE Initialize $M \gets 1$ and $\pi \gets k$
\FOR{$n = 1, \ldots, 2K$}
\STATE Initialize $r \gets 0$
\STATE \algorithmicif\ $n \geq 2k$\ \algorithmicthen\ $\pi \gets \pi + 1 \ (\text{mod } K)$\ \algorithmicendif \COMMENT{sequential hopping}
\STATE \algorithmicfor\ $T_c$ time steps\ \algorithmicdo\ Pull $\pi$ and update $r \gets r + r_\pi (t)$\ \algorithmicendfor
\STATE \algorithmicif\ $r = 0$\ \algorithmicthen\ $M \gets M + 1$\ \algorithmicendif \COMMENT{increases if $T_c$ collisions}
\ENDFOR
\RETURN $M$
\end{algorithmic}
\end{algorithm}
\vspace{-0,5cm}
\caption{\label{alg2:estimm}estimate $M$ for the No Sensing setting.}
\vspace{-0,3cm}
\end{pseudocode}

\subsubsection{Exploration phases}
\label{sec:explnosens1}

Each exploration phase is split into two parts. During the first one, each player fixes to an arm following Musical Chairs procedure. After this procedure, players are in an orthogonal setting and can thus start the second part, where they hop sequentially and explore the active arms without any collision.  The decisions for accepting/rejecting arms are still based on the exploration pulls as in \algoone[.]The differences with the exploration of \algoone are the following:
\begin{enumerate}
\item  statistics are not shared among players; this induces an additional $M$ factor in the regret.
\item A Musical Chairs procedure is added at the beginning of a new exploration phase, if there was at least one declaration or fixation block in the previous communication phase. This procedure is needed to reach an orthogonal setting before the sequential hopping. Figures~\ref{fig:nosensblock1} and \ref{fig:nosensblock2} below illustrate when such a procedure is added. It corresponds to lines~\ref{alg2:explMC1}-\ref{alg2:explMC2} in Algorithm~\ref{algo:nosens1}.
\end{enumerate}

\begin{figure}[h!]
    \centering
   \resizebox{\linewidth}{!}{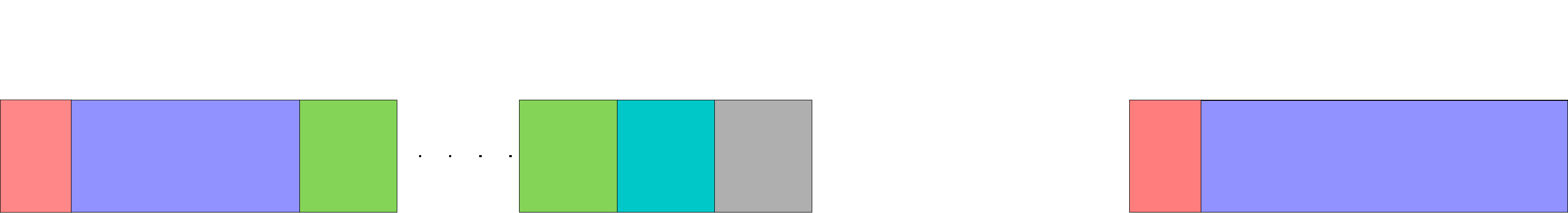}
    \caption{Alternating between fixation, exploration and declaration blocks. Case where a player declares sub-optimal arms and tries to occupy (without success) optimal arms.}
    \label{fig:nosensblock1}
    \end{figure}
    
\begin{figure}[h!]
    \centering
   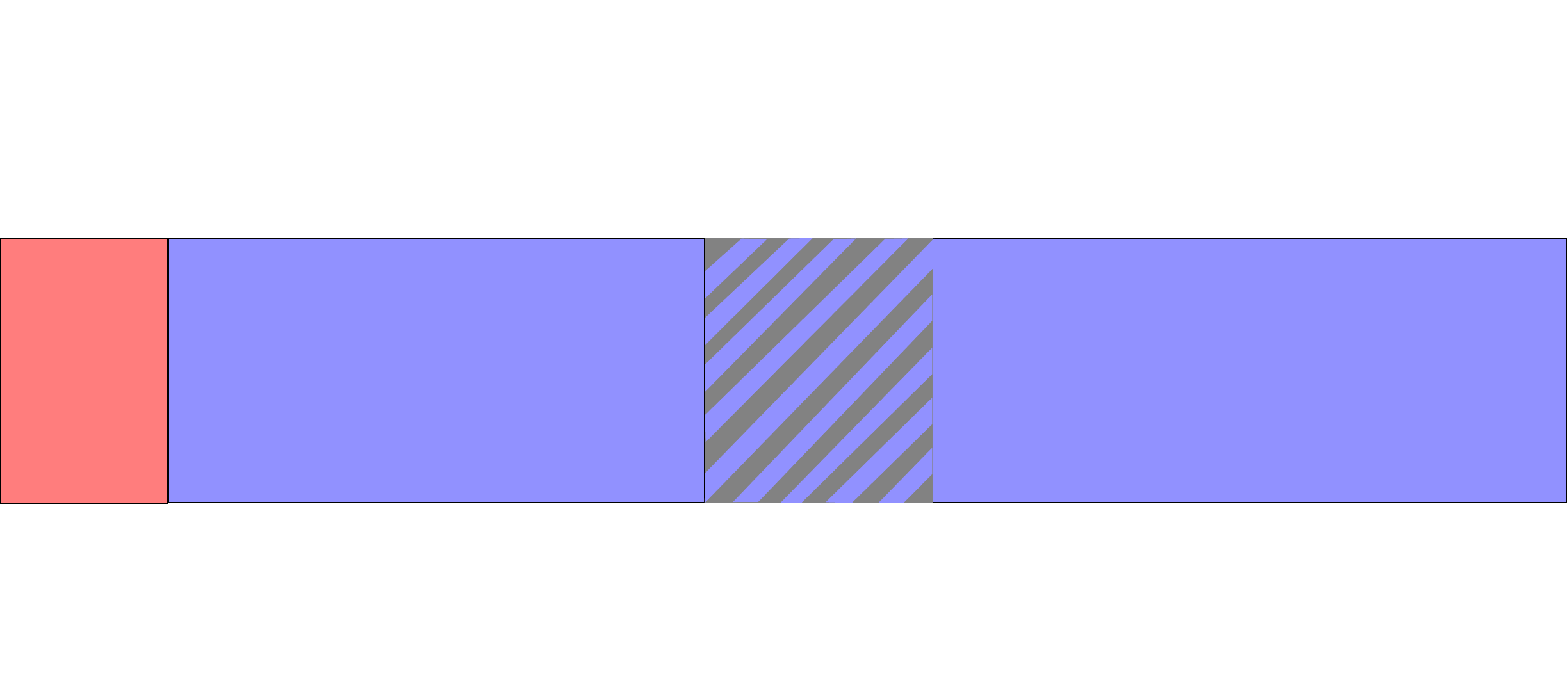
    \caption{Alternating between fixation, exploration and declaration blocks. Case with no declaration. In that case, the single declaration block, which just consists of sequential hopping, is included in the next exploration phase (lines~\ref{alg2:explMC20}-\ref{alg2:explMC2} in Algorithm~\ref{algo:nosens1}). No fixation phase is needed in that case.}
    \label{fig:nosensblock2}
    \end{figure}

\subsubsection{Communication phase}
\label{sec:declarebad}

Notice that all active players are in a communication phase at the same time. However, this phase is decomposed into blocks of same length $T_d$ (to keep synchronization). A block can be of three different types, and the type of a block does not need to be the same for all players, as illustrated in Figure~\ref{fig:nosensblock1}. Types are the following:
\begin{description}
\item[{Declaration block}] for player $j$: she communicates to other players that an arm is sub-optimal.
\item[{Fixation block}] for player $j$: she tries to occupy any arm that she detected as optimal. If she succeeds, she exploits that arm until the end.
\item[{Reception block}] for player $j$: she hops sequentially in order to detect other players' declarations.
\end{description}

Player $j$ starts the communication phase with declaration blocks, one per arm detected as sub-optimal\footnote{Of course, she does not declare any arm previously declared by another player.}. She then proceeds to a fixation block, had she detected any arm as optimal during the last exploitation phase.  She then proceeds to reception blocks until no new declaration is detected. As soon as no new declaration is detected, she starts a new exploration phase.

Notice that players keep receiving declarations from other players in any type of block.

\paragraph{Declaration block:}
In a declaration block, player $j$ follows \textbf{Declare Protocol}, described in Pseudocode~\ref{alg2:declare}. The idea is to frequently sample the sub-optimal arm in order to send a ``signal'' to the other players. They will  detect this signal by observing a significant loss in the empirical reward of this arm. However, a player sending a signal  should also be able to detect signals on other arms sent by other players. That is the reason why in order to declare an arm as sub-optimal, a player randomly chooses between pulling this arm and sequentially hopping.

\begin{pseudocode}[h!]
\vspace{-0,5cm}
\begin{algorithm}[H]
\caption*{\textbf{Declare Protocol}}
\hspace*{\algorithmicindent} \textbf{Input:}  $k$ (arm to declare), $T_d$ (time of block), $\pi$ (first arm to pull in sequential hopping), $\mathbf{S}$ and $\mathbf{T}$ (exploration statistics), $[K_p]$ (active arms) \\
\hspace*{\algorithmicindent} \textbf{Output:} $D$ (signaled arms in this block)
\begin{algorithmic}[1]
\STATE Initialize $\mathbf{s},\mathbf{t} \gets$ Zeros$(K)$
\FOR{$T_d$ time steps}
\STATE Pull arm $i = \begin{cases}
    k \text{ with probability } \frac{1}{2}\\
    \pi \text{ with probability } \frac{1}{2}
    \end{cases}$
\STATE $s[i] \gets s[i] + r_i(t)$; $t[i] \gets t[i] + 1$ and $\pi \gets \pi +1 \ (\textrm{mod } [K_p])$
\ENDFOR 
\STATE $d \gets$ set of active arms $i$ verifying $\left| \frac{S[i]}{T[i]} - \frac{s[i]}{t[i]} \right| \geq \frac{S[i]}{4T[i]}$
\RETURN $d \cup \lbrace k \rbrace$ \COMMENT{arms signaled during the block}
\end{algorithmic}
\end{algorithm}
\vspace{-0,5cm}
\caption{\label{alg2:declare} Declare arm $k$ as sub-optimal.}
\end{pseudocode}
\medskip
Lemma~\ref{lemma:nosens1detectfix} gives an appropriate choice for $T_d$ such that the declaration is detected by every player, without detecting any false positive declaration, no matter the block they are currently proceeding, with high probability.

Let $\hat{\mu}_i$ and $\hat{r}_i$  be respectively the empirical reward during the exploration phases and during the last communication block for arm $i$ and player $j$. Arm $i$ is detected as signaled, \ie another player is declaring or exploiting this arm if:
\begin{equation}
\label{eq:detectdeclare}
|\hat{\mu}_i - \hat{r}_i| \geq \frac{\hat{\mu}_i}{4}.
\end{equation}

Lemma~\ref{lemma:nosens1detectfix} states that players will only detect arms declared as sub-optimal or exploited by a player with high probability. However, using the last reception block where there is no new signal, it is easy to distinguish  exploited arms from declared ones. Indeed, for the exploited arms, only $0$ are observed during this last block; while for the declared ones, no player pulls it except for the sequential hopping. At least a positive reward will thus be observed with probability $1- \frac{1}{T}$ on them during the block, thanks to its length $T_d$, which depends on $\mu_{\min}$.

\paragraph{Fixation block:}

In a fixation block, player $j$ proceeds to \textbf{Occupy Protocol}, described in Pseudocode~\ref{alg2:occupy}. She sequentially hops on the active arms and starts exploiting an optimal arm as soon as it returns a positive reward (\ie without collision at this step). In that case, she pulls this arm until the final horizon $T$. At the end of a block, if she did not manage to exploit any detected optimal arm, then all of them are occupied by other players with high probability thanks to the length of the block. Signals of other players are detected following the rule of Equation~\eqref{eq:detectdeclare}. 

\begin{pseudocode}[h!]
\vspace{-0,5cm}
\begin{algorithm}[H]
\caption*{\textbf{Occupy Protocol}}
\hspace*{\algorithmicindent} \textbf{Input:} $A$ (set to occupy), $T_d$ (time of block), $\pi$ (first arm to pull), $\mathbf{S}$ and $\mathbf{T}$ (exploration statistics), $[K_p]$ (active arms) \\
\hspace*{\algorithmicindent} \textbf{Output:} \texttt{Fixed} (exploited arm) , $D$ (signaled arms in this block)
\begin{algorithmic}[1]
\STATE Initialize $\mathbf{s},\mathbf{t} \gets$ Zeros$(K)$; $\texttt{Fixed} \gets -1$
\FOR{$T_d$ time steps}
\STATE \algorithmicif\ \texttt{Fixed} $=-1$ \algorithmicthen
\STATE \hspace*{\algorithmicindent} Pull $\pi$
\STATE \hspace*{\algorithmicindent} \algorithmicif\ $\pi \in A$ and $r_\pi (t) > 0$\ \algorithmicthen\ \texttt{Fixed} $\gets \pi$ \algorithmicendif \COMMENT{no collision on optimal arm}
\STATE \hspace*{\algorithmicindent} $s[\pi] \gets s[\pi] + r_{\pi}(t)$; $t[\pi] \gets t[\pi] + 1$ and $\pi \gets \pi +1 \ (\textrm{mod } [K_p])$
\STATE \algorithmicelse\ Pull \texttt{Fixed}\ \algorithmicendif
\ENDFOR 
\STATE $d \gets$ set of active arms $k$ verifying $\left|\frac{S[k]}{T[k]} - \frac{s[k]}{t[k]}\right| \geq \frac{S[k]}{4T[k]}$
\RETURN (\texttt{Fixed}, $d$)
\end{algorithmic}
\end{algorithm}
\vspace{-0,5cm}
\caption{\label{alg2:occupy}Try to start exploiting an arm among $A$.}
\vspace{-0,3cm}
\end{pseudocode}

\paragraph{Reception block:}

In a reception block, player $j$ sequentially hops and detects the signals of other players, following the rule of Equation~\eqref{eq:detectdeclare}. This corresponds to \textbf{Receive Protocol}, described in Pseudocode~\ref{alg2:detect}.

\begin{pseudocode}[h!]
\vspace{-0,5cm}
\begin{algorithm}[H]
\caption*{\textbf{Receive Protocol}}
\hspace*{\algorithmicindent} \textbf{Input:} $T_d$ (time of block), $\pi$ (first arm to pull), $\mathbf{S}$ and $\mathbf{T}$ (exploration statistics), $[K_p]$ (active arms) \\
\hspace*{\algorithmicindent} \textbf{Output:} $D$ (signaled arms in this block), $\mathbf{s}$ and $\mathbf{t}$ (statistics of the block)
\begin{algorithmic}[1]
\STATE Initialize $\mathbf{s},\mathbf{t} \gets$ Zeros$(K)$
\FOR{$T_d$ time steps}
\STATE Pull $\pi$
\STATE Update $s[\pi] \gets s[\pi] + r_{\pi}(t);\ t[\pi] \gets t[\pi] + 1$ and $\pi \gets \pi +1 \ (\textrm{mod } [K_p])$
\ENDFOR
\STATE $d \gets$ set of active arms $k$ verifying $\left|\frac{S[k]}{T[k]} - \frac{s[k]}{t[k]}\right| \geq \frac{S[k]}{4T[k]}$
\RETURN  $(d, \mathbf{s}, \mathbf{t})$
\end{algorithmic}
\end{algorithm}
\vspace{-0,5cm}
\caption{\label{alg2:detect}Detect other players' declarations (and wait).}
\vspace{-0,3cm}
\end{pseudocode}

\medskip
Notice that every active player will at least proceed to one reception block per communication phase (if she does not end up occupying an optimal arm). The last reception block is considered as the block where no new signal is detected. This block is thus the same for all active players with high probability. Moreover, the arms giving $0$ reward during this last reception block are the optimal arms exploited by other players. This allows to distinguish exploited arms (which are optimal) from declared ones (which are sub-optimal). This distinction is described in Pseudocode~\ref{alg2:update}. As a consequence, active players share a common set of active arms $[K_p]$ and number of active players $M_p$ at the end of each communication phase. 

\begin{pseudocode}[h!]
\vspace{-0,5cm}
\begin{algorithm}[H]
\caption*{\textbf{Update}}
\hspace*{\algorithmicindent} \textbf{Input:} Decl (declared arms), $s$ (statistics of last reception block), $[K_p]$ (set of active arms), $M_p$ (number of active players) \\
\hspace*{\algorithmicindent} \textbf{Output:} $[K_{p+1}]$ (updated set of active arms), $M_{p+1}$ (updated number of active players)
\begin{algorithmic}[1]
\STATE Opt $\gets \lbrace i \in \text{Decl } | \ s[i] = 0 \rbrace$
\STATE $[K_{p+1}] \gets [K_p] \setminus \text{Decl}$ and $M_{p+1} \gets M_p - \texttt{length}(\text{Opt})$
\RETURN  $([K_{p+1}],\ M_{p+1})$
\end{algorithmic}
\end{algorithm}
\vspace{-0,5cm}
\caption{\label{alg2:update}Update the active sets at the end of a communication phase.}
\vspace{-0,3cm}
\end{pseudocode}

\medskip

The complete description of \algotwo is given in Algorithm~\ref{algo:nosens1} below.

\begin{algorithm}[h!]
\caption{\label{algo:nosens1} \algotwo algorithm}
\hspace{\algorithmicindent} \textbf{Input:}  $T$ (horizon), $\mu_{\min}$ (lower bound of means)

\begin{algorithmic}[1]
\STATE \textbf{Initialization Phase:}
\STATE Set $T_c \gets \lceil \frac{\log(T)}{\mu_{\min}}\rceil$; $\pi \gets \text{\musicalchair}([K], K T_c)$
\STATE $M_p \gets \text{\estimmnosens}(\pi, T_c)$
\STATE Initialize $T_0 \gets \lceil \frac{2400 \log(T)}{\mu_{\min}} \rceil$; Decl $\gets \emptyset$; $T_d \gets 0$; $[K_p] \gets [K]$ and $\mathbf{S}, \mathbf{T}, \mathbf{s}, \mathbf{t} \gets \text{Zeros}(K)$

\vspace{0.2cm}
\FOR{$p = 1, \ldots, \infty$}
\vspace{0.2cm}

\STATE \textbf{Exploration Phase:}
\STATE $T_{\text{expl}} \gets K_p 2^p T_0$

\STATE \algorithmicif\ \texttt{length}(Decl)$>0$ \algorithmicthen \label{alg2:explMC1} \COMMENT{there was a declaration in the previous phase so players need\\ \hfill to reach an orthogonal setting among the new set of active arms}
\STATE \hspace*{\algorithmicindent}  $\pi \gets \text{\musicalchair}\left( [K_p], K_pT_c \right)$
\STATE \algorithmicelse\ \label{alg2:explMC20} $\mathbf{S} \gets \mathbf{S} + \mathbf{s}$; $\mathbf{T} \gets \mathbf{T} + \mathbf{t}$ and $T_{\text{expl}} \gets T_{\text{expl}} - T_d$ \COMMENT{statistics of the last reception block}
\STATE \algorithmicendif \label{alg2:explMC2}

\FOR[start exploration]{$T_{\text{expl}}$ steps}
\STATE Pull $\pi$;\ $S[\pi] \gets S[\pi] + r_\pi (t)$; $T[\pi] \gets T[\pi] + 1$ and $\pi \gets \pi + 1 \ (\text{mod } [K_p])$
\ENDFOR
\vspace{0.2cm}

\STATE \textbf{Communication Phase:} \COMMENT{$B_s = \sqrt{\frac{2 \log(T)}{s}}$ here}
\STATE Initialize $T_d \gets K_p T_0$ and Decl as empty set
\STATE  Rej $\gets$ set of active arms $k$ verifying $\#\{ i \in [K_p]\, \big|\, \frac{S[i]}{T[i]} - B_{T[i]} \geq \frac{S[k]}{T[k]} + B_{T[k]} \} \geq M_p$
\STATE  Acc $\gets$ set of active arms $k$ verifying $\# \{i \in [K_p]\, \big|\, \frac{S[k]}{T[k]} - B_{T[k]} \geq \frac{S[i]}{T[i]} + B_{T[i]} \} \geq K_p-M_p$\WHILE[declaration blocks]{$\text{Rej} \setminus \text{Decl} \neq \emptyset$}
\STATE Let $k \in \text{Rej} \setminus \text{Decl}$
\STATE $d \gets \text{Declare}(k, T_d, \pi, \mathbf{S}, \mathbf{T}, [K_p])$ and add $d$ to Decl
\ENDWHILE
\IF[fixation block]{$\text{Acc} \setminus \text{Decl} \neq \emptyset$}
\STATE $(\texttt{Fixed}, d) \gets \text{Occupy}(\text{Acc} \setminus \text{Decl}, T_d, \pi, \mathbf{S}, \mathbf{T}, [K_p])$ and add $d$ to Decl
\STATE \algorithmicif\ \texttt{Fixed} $\neq -1$\ \algorithmicthen\ go to line~\ref{algo2:exploit} (break)\ \algorithmicendif
\ENDIF
\STATE $d \gets \{0\}$
\WHILE[reception blocks]{$d \neq \emptyset$}
\STATE $(d, \mathbf{s}, \mathbf{t}) \gets \text{Receive}(T_d, \pi, \mathbf{S}, \mathbf{T}, [K_p])$ \COMMENT{$\mathbf{s}$ and $\mathbf{t}$ are the statistics}
\STATE $d \gets d \setminus \text{Decl}$ and add $d$ to Decl \COMMENT{so $d$ contains only the new signals.}
\ENDWHILE
\vspace{0.2cm}

\STATE \textbf{Update Statistics:}
\STATE $([K_p], M_p) \gets \text{Update}(\text{Decl}, s, [K_p], M_p)$ 

\vspace{0.2cm}
\ENDFOR
\vspace{0.2cm}

\STATE\textbf{Exploitation Phase:} Pull \texttt{Fixed} until $T$ \label{algo2:exploit}
\end{algorithmic}
\end{algorithm}
\newpage
\subsection{Regret analysis}
\label{app:alg2proofs}

This section is devoted to the proof of Theorem~\ref{thm:algo2}. It first proves several required lemmas. 

A decomposition of the regret similar to \algoone is used for \algotwo[:]$$R_T = R^{\text{init}} + R^{\text{explo}} + R^{\text{comm}}.$$
But in this section, a communication step is defined as a time step in a communication phase where there is at least a player using \textbf{Declare} or \textbf{Occupy protocol}, and $T_{\text{init}} \coloneqq 3KT_c$. Notice that the last reception block of a communication phase then counts as communication steps only if there were declarations in previous blocks of the communication phase. Otherwise, its statistics are indeed used for the arms estimation, as described in Algorithm~\ref{algo:nosens1}, lines~\ref{alg2:explMC20}-\ref{alg2:explMC2}, and it is then counted as exploration.

\subsubsection{Initialization regret}

The initialization phase lasts $3K T_c$ steps, so $R^{\text{init}} \leq 3MK T_c$. Lemma~\ref{lemma:nosensestimM} claims that the initialization is successful, meaning all players perfectly know $M$ after this phase,  with a probability depending on $T_c$ and justifies the choice $T_c = \lceil \frac{\log(T)}{\mu_{\min}}\rceil$.

\begin{lemm}
\label{lemma:nosensestimM}
With probability $1 - \mathcal{O} \left( MK \exp \left( - \mu_{\min} T_c \right) \right)$, at the end of the initialization phase, every player has a correct estimation of $M$ and players are in an orthogonal setting.\end{lemm}
\begin{proof}\label{proof:nosensestimM}
Similarly to the proof of Lemma~\ref{lemma:musicalchair}, the probability to encounter a positive reward for a player during the Musical Chairs procedure at time step $t$ is lower bounded by $\frac{\mu_{\text{min}}}{K}$. Hence using the same arguments, with probability $1 - \mathcal{O} \left( M\exp \left( - \mu_{\min} T_c \right) \right)$, all the players are pulling different arms after a time $K T_c$.

We now consider the \textbf{\estimmnosens protocol}. Every time a player sends a bit to another player, it will be detected. Let us now bound the probability that a player detects a ``collision'' with another player while there is not. This is the case when she encounters $T_c$ successive zero rewards on an arm while she is the only player pulling it. This happens with probability smaller than $\exp \left( -\mu_{\text{min}} T_c \right)$ for a single player in a single block. The union bound over the $M$ players and the $2K$ blocks yields the results.
\end{proof}

\subsubsection{Communication regret}

Lemma \ref{lemma:nosens1detectfix} provides the properties  and regret of the algorithm during the communication phases.
\begin{lemm}
\label{lemma:nosens1detectfix}
Let the length of a block be such that $ T_d = \lceil \frac{2400 K_p \log(T)}{\mu_{\min}} \rceil$,
then conditionally on the successful outcome of all the previous Musical Chairs procedures:
\begin{enumerate}
\item with probability $ 1 - \mathcal{O}\left( \frac{M}{T} \right)$, a player $j$ declaring an arm $i$ as sub-optimal will be successfully detected by all active players;
\item with probability $ 1 - \mathcal{O}\left( \frac{MK}{T} \right)$, no player will detect a false signal during the declaration block (i.e., no arm is detected as declared if there was no declaration or if it is not occupied by an active player);
\item with probability $1 - \mathcal{O}\left( \frac{M}{T} \right)$, if player $j$ starts occupying arm $k$, it is detected as a declaration by  all active players (following the rule of Equation~\eqref{eq:detectdeclare}).
\end{enumerate}

Thus, with probability $1 - \mathcal{O} \left( MK \exp(- \mu_{\min} T_c) + \left( K + \log(T) \right) \frac{KM}{T}\right)$, all communication phases are successful, \ie all signals are correctly detected and no false signal is detected. Then
\begin{equation*}
R^{\mathrm{comm}} = \mathcal{O} \left( \frac{M K^2}{\mu_{\min}} \log(T) \right).
\end{equation*}
\end{lemm}

\begin{proof}
We first prove the three points conditionally on the success of the previous Musical Chairs procedures.
\begin{itemize}
\item[1)] We prove this point in the more general case where the declaration of an arm $i$ follows the sampling: 
$ \begin{cases}
    \text{Pull } i \text{ with probability } \lambda_d,\\
    \text{Sample according to the sequential hopping on } [K_p] \text{ otherwise.}
    \end{cases}$

First, denote by $T_{i'}^{j'}$ the number of pulls by player $j'$ on arm $i'$ during a block of length $T_d$. Using the Chernoff bound,
\begin{equation}
\label{eq:chernoffnosens10}
\begin{aligned}
\mathbb{P} \left[ T_{i'}^{j'} \leq \frac{(1-\lambda_d) T_d}{2 K_p} \right]  & \leq \exp \left( -\frac{(1-\lambda_d) T_d}{8 K_p} \right), \\
&\leq \frac{1}{T} \qquad \text{as long as } \frac{(1-\lambda_d) T_d}{8 K_p} \geq \log(T).
\end{aligned}
\end{equation}

This last condition holds with $T_d$ chosen as in Equation~\eqref{eq:choiceTd} and this inequality holds, no matter the type of block player $j'$ is proceeding.

With probability $1 - \mathcal{O} \left( \frac{KM}{T} \right)$, all the $T_i^j$ are thus greater than $\frac{(1-\lambda_d) T_d}{2 K_p}$. This is also the case with probability $1$ for the exploration pulls as the first exploration phase is of length $T_d$. Let $r_i$ and $\hat{r}_i$ respectively denote the expected and the empirical observed rewards of the arm $i$ during this declaration phase for player $j$. Assume that the arm $i$ is declared as sub-optimal by another player during the considered phase. It then holds that $r_i \leq \mu_i (1-\lambda_d)$.

\begin{equation}
\label{eq:choiceTd}
\text{With the specific choice of} \hspace*{1.5cm} T_d \geq \frac{300 K_p \log(T)}{(1-\lambda_d) \lambda_d^2 \mu_{\text{min}}},
\end{equation}
Chernoff bound provides the following inequalities, conditionally on $T_i^j \geq \frac{(1-\lambda_d) T_d}{2 K_p}$,
\begin{align*}
\mathbb{P} \left[ |\hat{r}_i - r_i| \geq  \frac{\lambda_d \mu_i}{5} \right] \leq \frac{2}{T}\\
\text{and } \mathbb{P} \left[ |\hat{\mu}_i - \mu_i| \geq  \frac{\lambda_d \mu_i}{5} \right] \leq \frac{2}{T}& \qquad \text{for the exploration phases}.
\end{align*}

We then consider the high probability event $
|\hat{r}_i - r_i| \leq \frac{\lambda_d}{5} \mu_i
\text{ and }
|\hat{\mu}_i - \mu_i| \leq \frac{\lambda_d}{5} \mu_i
$.

As $\lambda_d \leq 1$, the second inequality yields
$
\frac{5}{6} \hat{\mu}_i \leq \mu_i \leq \frac{5}{4} \hat{\mu}_i.
$

If $i$ is declared by a player, $\mu_i - r_i \geq \lambda_d \mu_i$ and
\begin{equation}
\label{eq:nosensdetectrule}
\begin{aligned}
	|\hat{\mu}_i - \hat{r}_i| & \geq \frac{3 \lambda_d}{5} \mu_i  \geq \frac{\lambda_d}{2} \hat{\mu}_i.
\end{aligned}
\end{equation}


This means that with the detection rule described in Appendix~\ref{sec:declarebad} for $\lambda_d = \frac{1}{2}$, for a single arm $i$ and player $j$, with probability $1 - \mathcal{O} \left( \frac{1}{T} \right)$, the player will correctly detect the declaration of arm $i$ as sub-optimal by (at least) another player. 

\item[2)] As in the first point, with probability $1 - \mathcal{O} \left( \frac{1}{T} \right)$, it holds $|\hat{r}_i - r_i| \leq \frac{\lambda_d}{5} \mu_i$. The case of neither declaration nor exploitation by any other player actually corresponds to $r_i = \mu_i$. Thus we can rewrite Equation~\eqref{eq:nosensdetectrule} of the first case into
$
	|\hat{\mu}_i - \hat{r}_i| \leq \frac{2 \lambda_d}{5} \mu_i  \leq \frac{\lambda_d}{2} \hat{\mu}_i,
$ which holds with probability $1 - \mathcal{O} \left( \frac{1}{T} \right)$.
Considering all arms and players yields the second point.

\item[3)]
The same argument as in Lemma~\ref{lemma:nosensestimM} gives that with probability $1 - \mathcal{O}\left( \frac{1}{T} \right)$, player $j$ will actually starts occupying $k$ after a time $t_{\text{fix}}\leq \frac{K_p \log(T)}{\mu_{\text{min}}} \leq \frac{T_d}{2400}$.

Chernoff bound then provides a bound on the total reward $X_k^{j'}$ observed by $j'$ on $k$ for a time $t_{\text{fix}}$, i.e., for $T_k^{j'} \leq \frac{\log(T)}{\mu_{\text{min}}} $ pulls on $k$. 
\begin{equation}
\label{eq:detectfix1}
\mathbb{P} \left[ X_k^{j'} \geq 4 \mu_k \frac{\log(T)}{\mu_{\text{min}}} \right] \leq \exp \left( -\frac{3 \mu_k \log(T)}{3 \mu_{\text{min}}}\right) \leq \frac{1}{T} .
\end{equation}

Thus, Equation~\eqref{eq:detectfix1} claims that with probability $1 - \mathcal{O}\left(\frac{1}{T}\right)$, $X_k^{j'} \leq \frac{4 \mu_k \log(T)}{\mu_{\text{min}}}$.

However, $k$ will be occupied after that point and no other positive reward will be observed by player $j'$. As a consequence, her empirical reward on $k$ will be for this block
$\hat{r}_k^{j'}  \leq \frac{2 X_k^{j'} K_p}{(1-\lambda_d) T_d}\leq \frac{2\mu_k}{75}$. Using the same argument as in points 1) and 2), this guarantees $|\hat{\mu}_k^{j'} - \hat{r}_k^{j'}| \geq \frac{\hat{\mu}_k^{j'}}{4}$ and the result follows.
\end{itemize}

\medskip

Conditionally on the success of the previous Musical Chairs procedures (\ie players end these procedures in orthogonal settings), these three points imply that, with probability $1 - \mathcal{O}\left( \frac{MK}{T} \right)$, the communication block will be successful: all declarations are correctly detected, all detected optimal arms are exploited by a player and there is no false detection.

Let $N$ be the total number of exploration phases. By construction of the algorithm, $N \leq \lceil \log_2(T) \rceil$.
Also there can not be two different blocks used to declare or occupy the same arm, conditionally on the success of the previous communication blocks and Musical Chairs. Hence, conditionally on this event, there will be at most $N + K$ communication blocks, each succeeding with probability $1 - \mathcal{O} \left( \frac{KM}{T} \right)$ and there will be at most $K$ Musical Chairs procedures, each also succeeding with probability $1 - \mathcal{O}(M \exp(-\mu_{\min}  T_c))$. Using a chain rule argument, all the communication protocols and Musical Chairs procedures are successful with probability $1 - \mathcal{O}\Big( KM \exp(-\mu_{\min} T_c)+ ( K + N ) \frac{KM}{T} \Big)$ and the length of Comm is at most $\mathcal{O}\left(\frac{K^2 \log(T)}{\mu_{\text{min}}}\right)$, since only the communication phases with at least a Declaration or Fixation block are counted. This leads to the bound on $R^{\text{comm}}$ given by Lemma~\ref{lemma:nosens1detectfix}.
\end{proof}

\subsubsection{Exploration regret}
Conditionally on the success of the initialization phase, all the communication phases and all the Musical Chairs procedures at the beginning of exploration phases, the exploration (except the Musical Chairs) will be collision-free. Using similar arguments as in Lemma~\ref{lemma:exploregretopt}, we provide an upper bound for the exploration regret of \algotwo[.]

\begin{lemm}
\label{lemma:nonsensexplosynch} With probability $1 - \mathcal{O} \left( KM \exp\left( - \mu_{\min} T_c \right) + (K+ \log(T)) \frac{KM}{T} \right)$,
\begin{equation*}
R^{\mathrm{explo}} \leq \mathcal{O} \left( \sum\limits_{k > M} \min\bigg\lbrace \frac{M \log(T)}{\mu_{(M)} - \mu_{(k)}} , \sqrt{MT \log(T)}\bigg\rbrace + \frac{MK^2}{\mu_{\min}} \log(T)\right).
\end{equation*}
\end{lemm}
\begin{proof}
\label{proof:nonsensexplosynch}

First, as already claimed in the proof of the communication regret, the initialization, all the communication blocks and Musical chairs procedures succeed and there are at most $K$ Musical Chairs procedures during the exploration with probability $1 - \mathcal{O} \Big( MK \exp(- \mu_{\min} T_c) + ( K + \log(T) ) \frac{KM}{T}\Big)$. The remaining of the proof is conditioned on this event. A single Musical Chairs procedure lasts a time $\frac{K_p \log(T)}{\mu_{\text{min}}}$, hence the total regret incurred by the Musical Chairs is smaller than $\frac{MK^2 \log(T)}{\mu_{\text{min}}}$.

We now consider the regret of exploration without the Musical Chairs.
We denote by $N$ the number of exploration phases that will be run and the same notation as in Appendix~\ref{proof:explo} concerning $\Delta_k$. As the exploration phases are collision-free (conditionally on the success of initialization, communication and Musical Chairs), the Hoeffding inequality still holds: $\mathbb{P}\Big[\exists p \leq n : |\hat{\mu}_k(p) - \mu_k| \geq \sqrt{\frac{2\log(T)}{T_k(p)}} \ \Big] \leq \frac{2n}{T}$.

Since the players do not share their statistics, it can be shown with the same arguments as in Appendix~\ref{proof:explo} that a sub-optimal arm $k$ will be found sub-optimal with probability at least $1 - \mathcal{O} \left( \frac{NM}{T} \right)$ after $t_k = \mathcal{O} \left( \frac{\log(T)}{\Delta_k^2} \right)$ exploration pulls \textbf{for a single player} without being found optimal by any player before. Since the exploration phases are collision-free, the cost for pulling the sub-optimal arm $k$ is $\mathcal{O} \left( \min \left\lbrace \frac{M \log(T)}{\Delta_k}, \Delta_k T \right\rbrace \right)$.

The same reasoning as in Appendix~\ref{proof:exploregretopt} shows that the exploration regret due to non pulls of optimal arms is in $\mathcal{O} \left( \sum\limits_{k>M} \min\left\lbrace \frac{M \log(T)}{\mu_{(M)} - \mu_{(k)}} , \sqrt{MT \log(T)}\right\rbrace\right)$ conditionally on correct estimations of the arms.

\medskip

As $N \leq \lceil \log_2(T) \rceil$, all those arguments yield the bound on $R^{\text{explo}}$, with probability $1 - \mathcal{O} \Big( \frac{KM \log(T)}{T} + KM \exp ( - \mu_{\min} T_c ) + (K+ \log(T)) \frac{KM}{T} \Big)$.
\end{proof}

Theorem~\ref{thm:algo2} can now be deduced from Lemmas~\ref{lemma:nosensestimM}, \ref{lemma:nosens1detectfix}, \ref{lemma:nonsensexplosynch} and Equation~\eqref{eq:regdec}. The total regret is upper bounded by the sum of the regrets mentioned in Lemmas~\ref{lemma:nosensestimM}, \ref{lemma:nosens1detectfix}, \ref{lemma:nonsensexplosynch} and the regret when a ``bad'' event occurs. According to these lemmas, the probability that a bad event may happen is indeed in $\mathcal{O} \left( (K+ \log(T)) \frac{KM}{T} \right)$. The average regret due to bad event is thus upper bounded by this probability multiplied by $MT$. This term is then dominated by the communication regret.
%
\section{On the inefficiency of \textsc{selfish} algorithm}
\label{sec:selfish}

A linear regret for the \textsc{selfish} algorithm in the No Sensing model has been recently conjectured \citep{besson}. This algorithm seems to have good results in practice, although rare runs with linear regret appear. This is due to the fact that with probability $p>0$ at some point $t$, both independent from $T$, some players might have the same number of pulls and the same observed average rewards for each arm. In that case, the players would pull the exact same arms and thus collide until they reach a tie breaking point where they could choose different arms thanks to a random tie breaking rule. However, it was observed that such tie breaking points would not appear in the experiments, explaining the linear regret for some runs. Here we claim that such tie breaking points might never happen in theory for the \textsc{selfish} algorithm when the rewards follow Bernoulli distributions, if we add the constraint that the numbers of positive rewards observed for the arms are all different at some stage. This event remains possible with a probability independent from $T$. 

\begin{prop}
\label{prop:ucb1}
For $s, s' \in \mathbb{N}$ with $s\neq s'$:
\begin{equation*}
\forall n \geq 2, t, t' \in \mathbb{N}, \qquad \frac{s}{t} + \sqrt{\frac{2\log(n)}{t}} \neq \frac{s'}{t'} + \sqrt{\frac{2\log(n)}{t'}}.
\end{equation*}
\end{prop}

\begin{proof}
First, if $t=t'$, these two quantities are obviously different as $s\neq s'$.

We now assume $\frac{s}{t} + \sqrt{\frac{2\log(n)}{t}} = \frac{s'}{t'} + \sqrt{\frac{2\log(n)}{t'}} \text{ with } t \neq t'$.\\
This means that $ \sqrt{\frac{2\log(n)}{t}} -\sqrt{\frac{2\log(n)}{t'}}$ is a rational, i.e., for some rational $p$, $\log(n)(t+t'-2\sqrt{tt'}) = 2p$.
\begin{align*}
\text{It then holds} \hspace*{3.4cm}\log(n)\sqrt{tt'} &= \log(n)\frac{t+t'}{2} - p, \\
tt' \log^2(n) &= \log^2(n)(\frac{t+t'}{2})^2 - p(t+t')\log(n) + p^2, \\
\log^2(n)(\frac{t-t'}{2})^2 - p(t+t')\log(n) + p^2 &= 0.
\end{align*}

Since $(\frac{t-t'}{2})^2 \neq 0$ and all the coefficients are in $\mathbb{Q}$ here, this would mean that $\log(n)$ is an algebraic number. However, Lindemann–Weierstrass theorem implies that $\log(n)$ is transcendental for any integer $n \geq 2$. We thus have a contradiction.
\end{proof}

The proof is only theoretical as computer are not precise enough to distinguish rationals from irrationals. The advanced arguments are not applicable in practice. Still, this seems to confirm the conjecture proposed by \cite{besson}: a tie breaking point is never reached, or at least not before a very long period of time. 


\end{document}